\documentclass[10pt]{article}
\usepackage[a4paper,margin=1in,footskip=0.25in]{geometry}
\usepackage[english]{babel}
\usepackage[utf8x]{inputenc}
\usepackage[T1]{fontenc}

\usepackage[ruled,linesnumbered]{algorithm2e}

\usepackage{mathcomSTEv4}
\usepackage[normalem]{ulem}
\usepackage{amsmath,amsthm}
\usepackage{amssymb}
\usepackage{graphicx}
\usepackage{bm,comment}
\usepackage[colorinlistoftodos]{todonotes}
\usepackage[colorlinks=true, allcolors=blue]{hyperref}
\usepackage{microtype}
\usepackage{graphicx}
\usepackage{subfigure}
\usepackage{booktabs} % for professional tables

\usepackage{hyperref}

\newcommand{\DO}{{\mathsf{do}}}

\newtheorem{theorem}{Theorem}

\newtheorem{lemma}{Lemma}

\title{Hierarchical Causal Bandit}

\author{%
  Ruiyang Song,\thanks{
  Department of Electrical Engineering,
  Stanford University, 
  \texttt{ruiyangs@stanford.edu}.} 
  \and 
  Stefano Rini,\thanks{
  Department of Electrical and Computer Engineering,
   National Chiao Tung University, 
   \texttt{stefano@nctu.edu.tw}.}
  \and
   Kuang Xu\thanks{
   Graduate School of Business,
  Stanford University, 
   \texttt{kuangxu@stanford.edu}.}
}

\begin{document}
\maketitle

\begin{abstract}

Causal bandit is a nascent learning model where an agent sequentially experiments in a causal network of variables, in order to identify the reward-maximizing intervention. 
Despite the model’s wide applicability, existing analytical results are largely restricted to a parallel bandit version where all variables are mutually independent. We introduce in this work the hierarchical causal bandit model as a viable path towards understanding general causal bandits with dependent variables. The core idea is to incorporate a contextual variable that captures the interaction among all variables with direct effects. 
Using this hierarchical framework, we derive sharp insights on algorithmic design in causal bandits with dependent arms and obtain nearly matching regret bounds in the case of a binary context.

\end{abstract}
\section{Introduction}

The causal bandit model is a sequential causal inference framework that aims to quickly identify the best intervention by adaptively choosing actions based on past observations.
A generic causal bandit problem consists of a finite set of discrete arms, each associated with a scalar value, and a reward that depends on the values of the arms. Pulling an arm corresponds to setting the said arm to be of a specific value, and such actions could influence both the reward (direct effect) and the values of other arms (indirect effect). A standard goal in this literature is to minimize the so-called simple regret for a given time horizon, defined as the gap between the expected reward from the optimal intervention and from the action chosen by the algorithm at the end of the horizon.  

Existing results on causal bandits tend to revolve around two main paradigms, distinguished by the complexity of the interactions among the arms. On the one end of the spectrum, a general formulation of the problem known as  structural causal bandit  \cite{bareinboim2015bandits,lee2018structural} allows arms to have complex dependencies and confounding effects, modeled by a directed acyclic causal network. On the other end, the parallel causal bandit formulation \cite{lattimore2016causal} considers a significantly simpler setting, where the default values of the arms admit a product distribution, and interventions on one arm does not have any causal effects on the value of the other arms (see Fig.~\ref{fig:network_topologies_2}). While the structural bandit formulation is more general, its inherent complexity makes rigorous analysis of regret difficult. In contrast, sharp upper and lower  bounds on simple regret have been derived for parallel causal bandits \cite{lattimore2016causal}. Perhaps more importantly, these bounds demonstrate a clear dependency on  a certain complexity measure associated with the value distribution that further sheds light on the key factors that determine the difficulty of adaptive causal inference.

The main goal of the present paper is to make progress towards bridging the gap between these  worlds: we want a richer  class of causal bandit models that goes beyond the parallel bandit version in terms of capturing more  realistic cross-arm dependencies, while having enough structure to allow for obtaining sharp analysis and generalizable insights. 

\paragraph{Summary of Main Contributions}  We introduce in this work the hierarchical causal bandit (HCB) model that uses a context variable to allow for complex dependencies across the arm values; an illustration of the model is given in Fig.~\ref{fig:network_topologies}. In a nutshell, the model consists of a top context node, $S$, whose value causally influences the distribution of values of the arms $X_1, \ldots, X_N$, which in turn determine the distribution of the reward, $Y$. The values of the arms are mutually independent conditional on a specific realization of the context node. In other words, HCB can be thought of as a convex combination of a finite set of sub-problems, each being a parallel causal bandit, mediated by the context variable. 

The hierarchical causal bandit straddles between the structural and parallel bandit models: on the one hand, as the alphabet size of the context variable gets large, HCB is rich enough to approximate arbitrary distributions across the arm values, thus endowing it with formidable modeling power as compared to parallel bandits; on the other hand, its relative simple structure clearly cannot (and is not intended to) capture complex causal effects in general networks. Nevertheless, we believe that HCB provides, at a minimum, a meaningful step towards understanding general causal bandits. Using a HCB model with binary interventions, we derive upper and lower bounds on the simple regret. While our results hold for any context node with a finite alphabet, the bounds are nearly sharp in the case where the context node is binary. 

Our analysis also pays special attention to elucidating what features of the conditional arm value distribution determine the regret scaling. 
One key insight is that the problem's hardness is determined by the worst sub-problem across all context states. This is surprising, since one would expect that the agent could leverage data gathered from an easy sub-problem to inform learning in a more difficult one, considering that the mapping  from arm values to reward is context-invariant. Our analysis shows that this is not the case, and in fact that the reward function is powerful enough to effectively ``discriminate'' against a specific sub-problem, rendering the aforementioned cross-context data-sharing ineffective.  
We subsequently use these insights to design an efficient learning policy that provably achieves the simple regret upper bound.

The remainder of this paper is organized as follows. 
   In Sec.~\ref{sec:Hierarchical Causal Bandit}, we formally introduce the hierarchical causal bandit model, with our main results to follow in Sec.~\ref{sec:main result}. 
   In Sec.~\ref{sec:upper} and Sec.~\ref{sec:lower}, we provide proof sketches for the main results, with a technical preliminary in Sec.~\ref{sec:Related Models} that reviews some of the relevant results in the literature.  
Finally, Sec.~\ref{sec:conclusion} concludes the paper.

\def\lw{0.7pt}

\begin{figure}[h!]
    \centering
    \begin{tikzpicture}[scale = 0.25]
        \node (zero) [line width=\lw, draw, circle,inner sep=0.15cm] at (10,5){$S$};
        \node (zero_comm) at (15,5.5){context};
        \node (one) [line width=\lw, draw, circle,inner sep=0.15cm] at (0,0){$X_1$};
        \node (two) [line width=\lw, draw, circle,inner sep=0.15cm] at (5,0){$X_2$};
        \node (three) [line width=\lw, draw, circle,inner sep=0.15cm] at (10,0){$X_3$};
        \node (line width=\lw, arm_comm) at (13,1){arm};
        % \node (four) [draw, circle,inner sep=0.15cm] at (15,0){$X_4$};
        \node (dots) at (15,0){$\ldots$};
        
        \node (last) [line width=\lw, draw,circle, inner sep=0.15cm] at 
        (20,0){$X_{N}$};
        \node (reward) [line width=\lw, draw, circle,inner sep=0.15cm] at (10,-5){$Y$};
        \node (reward_comm) at (15,-5.5){reward};
                
        \draw[line width=\lw,style={->,black}] 
        (zero) -- (one);
        \draw[line width=\lw,style={->,black}] 
        (zero) -- (two);
        \draw[line width=\lw,style={->,black}] 
        (zero) -- (three);
        %\draw[line width=\lw,style={->,blue}] 
        %(four) -- (reward);
        \draw[line width=\lw,style={->,black}] 
        (zero) -- (last);

        \draw[line width=\lw,style={->,black}] 
        (one) -- (reward);
        \draw[line width=\lw,style={->,black}] 
        (two) -- (reward);
        \draw[line width=\lw,style={->,black}] 
        (three) -- (reward);
        %\draw[line width=\lw,style={->,blue}] 
        %(four) -- (reward);
        \draw[line width=\lw,style={->,black}] 
        (last) -- (reward);
        
        % \draw[line width=\lw] (one) -- (two);
        % \draw[line width=\lw] (one) -- (three);
        % \draw[line width=\lw] (one) -- (four);
        % \draw[line width=\lw] (one) -- (five);
        % \draw[line width=\lw] (two) -- (three);
        % \draw[line width=\lw] (two) -- (four);
        % \draw[line width=\lw] (two) -- (five);
        % \draw[line width=\lw] (three) -- (four);
        % \draw[line width=\lw] (three) -- (five);
        % \draw[line width=\lw] (four) -- (five);
  
        % \node (one) [draw, circle,inner sep=0.25cm] at (0,0){};
        % \node (two) [draw, circle,inner sep=0.25cm] at (3.5,2.5){};
        % \node (three) [draw, circle,inner sep=0.25cm] at (7,0){};
        % \node (four) [draw, circle,inner sep=0.25cm] at (5.25,-3.5){};
        % \node (five) [draw, circle,inner sep=0.25cm] at (1.75,-3.5){};

    \end{tikzpicture}
    \caption{
    %An illustration of the hierarchical causal bandit (HCB) model. Here the directed edges indicate causal influence. 
    The hierarchical causal bandit (HCB) model. Directed edges indicate causal influence. }
    \label{fig:network_topologies}
\end{figure}
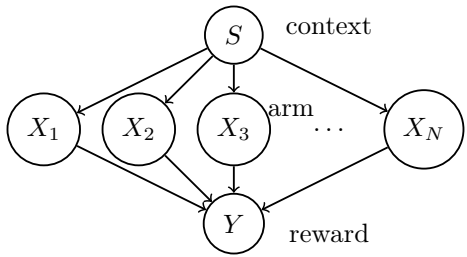

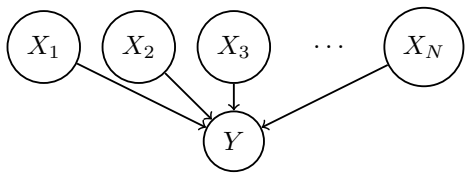
\begin{figure}[h!]
    \centering
    \begin{tikzpicture}[scale = 0.25]
  
        \node (one) [line width=\lw,draw, circle,inner sep=0.15cm] at (0,0){$X_1$};
        \node (two) [line width=\lw,draw, circle,inner sep=0.15cm] at (5,0){$X_2$};
        \node (three) [line width=\lw,draw, circle,inner sep=0.15cm] at (10,0){$X_3$};
        % \node (four) [draw, circle,inner sep=0.15cm] at (15,0){$X_4$};
        \node (dots) at (15,0){$\ldots$};
        
        \node (last) [line width=\lw,draw,circle, inner sep=0.15cm] at 
        (20,0){$X_{N}$};
        \node (reward) [line width=\lw,draw, circle,inner sep=0.15cm] at (10,-5){$Y$};

        \draw[line width=\lw,style={->,black}] 
        (one) -- (reward);
        \draw[line width=\lw,style={->,black}] 
        (two) -- (reward);
        \draw[line width=\lw,style={->,black}] 
        (three) -- (reward);
        %\draw[line width=\lw,style={->,blue}] 
        %(four) -- (reward);
        \draw[line width=\lw,style={->,black}] 
        (last) -- (reward);
        
        % \draw[line width=\lw] (one) -- (two);
        % \draw[line width=\lw] (one) -- (three);
        % \draw[line width=\lw] (one) -- (four);
        % \draw[line width=\lw] (one) -- (five);
        % \draw[line width=\lw] (two) -- (three);
        % \draw[line width=\lw] (two) -- (four);
        % \draw[line width=\lw] (two) -- (five);
        % \draw[line width=\lw] (three) -- (four);
        % \draw[line width=\lw] (three) -- (five);
        % \draw[line width=\lw] (four) -- (five);
  
        % \node (one) [draw, circle,inner sep=0.25cm] at (0,0){};
        % \node (two) [draw, circle,inner sep=0.25cm] at (3.5,2.5){};
        % \node (three) [draw, circle,inner sep=0.25cm] at (7,0){};
        % \node (four) [draw, circle,inner sep=0.25cm] at (5.25,-3.5){};
        % \node (five) [draw, circle,inner sep=0.25cm] at (1.75,-3.5){};

    \end{tikzpicture}
    \caption{The parallel causal bandit model \cite{lattimore2016causal}.}
    \label{fig:network_topologies_2}
\end{figure}

\subsection{Related work}

The causal bandit model is derived from the multi-armed bandit (MAB) problem by incorporating the causality framework of \cite{pearl2009causality}; see \cite{bubeck2012regret} for an overview of the MAB literature. 
In contrast to the conventional stochastic bandit model, where the reward obtained in each round depends solely on the identity of the chosen arm, the reward in a causal bandit problem depends on the values of all arms, which are typically assumed to be observable to the agent after the fact, thus allowing for substantially more efficient learning than one would expect in a generic MAB problem \cite{lattimore2016causal}.
To the best of our knowledge,   \cite{ortega2014generalized} was the  first to consider causal extensions of MAB, where they analyze an extension of Thompson sampling by treating actions as  causal interventions.
In particular, they argue that Thompson sampling remains a principled approach in choosing  arm pulls also in the presence of causal inference.
In a similar vein, the authors of \cite{bareinboim2015bandits} demonstrate that, in the presence of confounding variables, the value that a variable would have taken had it not been intervened on can provide important contextual information.

The current paper is perhaps most related to the program laid out in \cite{lattimore2016causal}, where they prove closed-form upper and lower bounds for simple regret   for the parallel bandit model. 
\cite{lattimore2016causal} also shows an upper bound for the simple regret under a general causal structure, although a matching lower bound has yet to be derived.
Following \cite{lattimore2016causal}, the agent in our model can only perform intervention on at most one variable per round. 

Generalizations of the parallel causal bandit models have been considered \cite{yabe2018causal, lee2018structural}. 
In~\cite{yabe2018causal}, the authors consider the scenario where an arbitrary set of intervention is made available at the agent. However, their upper bounds depend on the solution of an optimization problem, which cannot be solved analytically in general. \cite{lee2018structural} considers a general structural causal model and the problem of determining a minimal intervention set, but does not study the metric of simple regret. 

Another set of  interesting results are those in the context of \emph{causal graph discovery} which examines the number of controlled experiments required to discover a causal
graph \cite{hauser2012two}.
In \cite{eberhardt2010causal}, Eberhardt proposes the use of randomized strategies in causal graph discovery and shows that, if the designer is restricted to single-variable interventions, the worst case expected number of experiments required scales as the number of nodes in the graph. 
In \cite{hu2014randomized}, it is shown that randomization provides substantial improvements in  \cite{eberhardt2010causal}.

\noindent
{\bf Notation:}
Throughout this paper, we adopt the following notation. 
Given $\al \in [0,1]$, let $\overline{\al}=1-\al$.
For $x\in \mathbb R$, let $\lceil x \rceil = \min\{z\in \mathbb Z: z\geq x\}$ be the smallest integer that is greater than or equal to $x$.
Random variables are indicated with capital Roman letters, their support with calligraphic letters, e.g., $Z \in \Zcal$, while the support may not be indicated when implied by the context. 
Sets of random variables are denoted by bold capital Roman letters, e.g.,
 $\Zv=\{Z_n\}_{n \in [N]}$. For a subset $\mathcal B\subseteq[N]$, denote by $\mathbf{Z}_{\mathcal B} = \{Z_n\}_{n\in\mathcal B}$.
When there exists an ordering of the variables in the set we will, with some abuse of notation, use set and vector notations interchangeably.
The set of positive integers $\{i,i+1,\ldots,j\}$  for $i <j$ is denoted as $[i:j]$. When $i=1$, we further simplify the notation as $[j]$.
For $x\in \Rbb$ and $\delta \in \Rbb^+$, let $[x\pm \delta] = [x-\delta,x+\delta]$.
For an event $B$, let $\mathbf{1}_B$ be the indicator function that returns 1 if $B$ is true and 0 otherwise. 
We also adopt the following graph-theoretic notation: for a directed acyclic graph (DAG), $\mathcal G=(\Vcal,\Ecal)$ with $\Vcal=[N]$, $\Ecal \subseteq \Vcal \times \Vcal$, denote by $\pa(v)$ the index set of the parent nodes of $v  \in\Vcal$, i.e.,
\begin{equation}
    \pa(v) = \{n \in [N]: (n,v) \in\Ecal\}.
\end{equation} 
A Bayesian network is a probabilistic model built over a DAG in which the random variable $Z_n$ is associated with the node $n \in [N]$.
Accordingly, $\pa(Z_n)$ corresponds to the set of indices of the variables that are parents of $Z_n$.

\section{The Hierarchical Causal Bandit Model}
\label{sec:Hierarchical Causal Bandit}

We now introduce the system  model.\footnote{HCB is an instance of the structural causal bandit. 
For conciseness of presentation, we shall not present the structural causal bandit in its full generality: we refer the interested reader to \cite{bareinboim2015bandits,lattimore2016causal,lee2019structural} and references therein.}
A hierarchical causal bandit problem, illustrated in Fig.~\ref{fig:network_topologies}, consists of a set of random  variables $S, \Xv$ and $Y$. Here, $S$ is referred to as the \emph{context},  $Y$ the \emph{reward}, and $\Xv=\{X_n\}_{n \in [N]}$, where $X_n$ is the value of  the $n^{\rm th}$ \emph{arm}.
We assume that each variable takes only a finite number of values.  
The distribution of the variables is defined by an underlying DAG: there is a directed edge from $S$ to each $X_n$, ${n \in [N]}$, and there is a directed edge between each arm $X_n$ and the reward $Y$.
Accordingly, following \cite{jordan2003introduction}, we have that the  joint distribution can be factorized as 
\ea{
P_{S,\Xv, Y}=P_S \lb \prod_{n \in [N]}  P_{X_n|S} \rb P_{Y|\Xv}.
\label{eq:joint}
}

An \emph{intervention} in HCB corresponds to fixing the outcome of one of the random variables $Z \in \{S,\Xv\}$ to a prescribed outcome, i.e., setting $Z=z$ regardless of its natural distribution, denoted by $\DO(Z=z)$.
This induces a distribution in the observation which corresponds to the graph in which the edges incoming to the chosen variable are removed. 
The resulting distribution of $P_{S,\Xv,Y}$ is then obtained analogously to \eqref{eq:joint} but for $P_{Z|\pa(Z)}=\ones_{\{Z=z\}}$.
Let us denote the distribution of the set of random variables $\Wv$ under the intervention $Z=z$ as $P(\Wv| \DO(Z=z))$.
We will also consider the empty intervention denoted by $\DO(\varnothing)$, i.e., the case that the agent purely observes the variables in the system. Under the empty intervention, the distribution $P(\Wv| \DO(\varnothing))$ corresponds to the distribution of $\Wv$ under the law in \eqref{eq:joint}.
The set of possible interventions is denoted as $\Acal$. In the following we consider two settings for HCB:
\begin{enumerate}
    
    \item {\bf non-manipulable context (HCB-nmc)} in which the context cannot be intervened upon, i.e.,
    \begin{equation}
    \Acal^{\mathsf{nmc}}=\{\DO(X_n=x) ,\ n \in[N],\ x \in\Xcal\}\cup \{\DO(\varnothing)\},
    \label{eq:nmc}
\end{equation}

\item {\bf manipulable context (HCB-mc)} in which the context belongs to the set of possible interventions, just like the other arms, i.e.,
    \begin{equation}
    %\Acal^{\mathsf{mc}}=\{\DO(X_n)=x,\ n\in[N],\ x \in \Xcal\}\cup \{\DO(\varnothing)\} \cup \{ \DO(S)=s, s \in \Scal \},
    \Acal^{\mathsf{mc}}= \Acal^{\mathsf{nmc}} \cup \{ \DO(S=s), s \in \Scal \}.
    \label{eq:mc}
\end{equation}

\end{enumerate}

\medskip

In an HCB, an agent wishes to maximize the expected reward while not knowing the underlying distribution in \eqref{eq:joint}.
In order to do so, the agent  is able to choose an intervention  and observe the realizations of the variables in the graph. 
More precisely, the agent can choose $T$ interventions, for some time horizon $T$: after $T$ interventions, the goal of the agent is to  choose the action in $\Acal$ yielding the highest expected reward. 
Suppose an action $\Ah_T\in \Acal$ is identified as the best arm at the end of the $T$ steps of experiments. The agent then obtains a {\bf simple regret} at time $T$, defined as
\ea{
R_T=\mu^*-\sum_{a\in \mathcal A}\mu_{a}P(\hat{A}_T=a),
\label{eq:pure exploration regret}
}
where $\mu_{a} = \mathbb E[Y|A=a]$ is the expected reward when the agent chooses action $a$, $a\in \mathcal A$, and $\mu^*=\max_{a\in\mathcal A}  \mu_a$ is the expected reward of the optimal arm, with the maximum taken over all possible actions $a\in \Acal$. Here, $P(\hat{A}_T =a)$ is the probability that the agent identifies $a$ as the optimal arm after the experiments, which is determined by the distribution of the variables in the system and the algorithm adopted by the agent. 
In other words, $R_T$ measures the expected loss in choosing the intervention $\hat{A}_T$ instead of the optimal intervention $a^*=\argmax_{a\in\mathcal A} \ \Ebb\left[Y|A=a \right]$ that yields the highest reward $\mu^*$.

\subsection{Bernoulli HCB}
\label{sec:Bernoulli HCB}

Let us next restrict ourselves to a scenario where all variables are binary-valued: $\Scal=\Xcal=\Ycal=\{0,1\}$. 
Here, HCB is defined through the tuple $( \al,\pv,\qv,r)$, where
$\al = P(S=1)\in(0,1)$, $\pv=(p_1,\ldots,p_N)$, $\qv=(q_1,\ldots,q_N)$, and $\pv,\qv \in (0,1)^N$ are probability vectors such that $p_n=P(X_n=1|S=1)$ and $q_n=P(X_n=1|S=0)$. Also, $r: \  \{0,1\}^N \goes \Rbb^+$ is defined as $r(\xv)=P(Y=1|\Xv=\xv)$ and referred to as the reward function. We will denote the Bernoulli HCB with non-manipulable and manipulable context as Bernoulli HCB-nmc and Bernoulli HCB-mc, respectively. 

This choice of  binary-valued variables is motivated by two primary considerations. First, from a practical point of view, some of the most popular applications of causal inference and randomized experiments often involve binary-valued interventions (e.g., treatment vs.~control), thus making it a natural starting point. Second, our main results and conceptual insights are most easily seen through this binary model.

We should mention that we do extend the regret bounds to the case where the context variable $S$ has a general finite alphabet, while the other variables in the system are still binary in Thm.~\ref{thm:upper bound K}.  
In particular,
we have $\mathcal X=\mathcal Y=\{0,1\}$, and $|\mathcal S|=K$ where $K\in\mathbb Z_+$ is a fixed integer. We will refer to this model as the Bernoulli HCB with $K$-context, which is depicted by parameters $\left(\bm{\alpha},\{\mathbf{p}^{(i)}\}_{i\in[0:K-1]},r\right)$, where $\alpha_i = P(S=i)$, $i\in[0:K-1]$, and $p^{(i)} = (p_1^{(i)},\ldots,p_N^{(i)})$ with $p_j^{(i)} = P(X_j=1|S=i)$. We do not treat the case of non-binary arms (other than $S$) in this work, though we expect that our results can be generalized to non-binary arms in a future work with more refined analysis. 

\section{Main Results}\label{sec:main result}
\label{sec. Main contributions}
%
%In this section we introduce the main contribution of this paper. 
%
The main contribution of the paper is the derivation of upper and lower bounds on the simple regret  of  HCB.
Proof sketches for these bounds are provided in the latter sections, Sec. \ref{sec:upper} and Sec. \ref{sec:lower}, while detailed proofs are available in the appendix.

\label{notationr the hierarchical causal bandit}

To better understand what makes a causal bandit problem hard (large simple regret), it is instructive to first highlight the following measure of difficulty, first used by \cite{lattimore2016causal} for studying parallel causal bandits, which will play a central role in our analysis as well. For a probability vector $\vv$, define $m(\vv)$ to be a value that characterizes the number of entries in $\vv$ that are highly biased towards 0 or 1: 
\begin{equation}\label{def:m}
    m(\vv) = \min\left\{s\in\Rbb^+:s\geq\left|\mathcal I_s(\vv)\right|\right\},
\end{equation}
where
\ea{
\mathcal I_s(\vv)=\left\{n \in [N]:\min \{v_n,\vo_n\}<\frac{1}{s}\right\}.
}
The significance of $m$ is as follows: when applied to a Bernoulli parallel causal bandit, $\vv$ can be thought of as the probability vector representing the distributions of the $N$ arms.  Then, one can imagine that the causal effect of an arm that is heavily biased tend to be more difficult to learn, simply because their values rarely change on their own accord and learning would thus require the more expensive manual intervention. Indeed, \cite{lattimore2016causal} shows that $m(\vv)$ essentially determines the scaling of simple regret in parallel causal bandits. 

In the case of hierarchical causal bandits, it is no longer obvious whether $m$ should still play a role, and if so, in what manner. One plausible conjecture based on the insights from parallel bandits might suggest that the simple regret should depend on $m(\vv)$ where $\vv$ is the {marginal} distribution of the $N$ arms. As eluded to in the Introduction, the logic goes that an arm should be easy to learn as long as it is sufficiently stochastic under {some} context. Our results demonstrate, however, that this is not the case.

\begin{theorem}[Bernoulli HCB simple regret upper bound]
\label{thm:upper bound}
Consider a Bernoulli HCB-nmc as in Sec.~\ref{sec:Bernoulli HCB}. There exists an algorithm $\phi_{\mathsf{nmc}}$ such that
\begin{equation}
    R_T^{\mathsf{nmc}} \leq 122\sqrt{\lb m(\pv)\al+m(\qv) \alo \rb \f {\log NT}{T}},
\label{eq:upper bound nmc}
\end{equation}
for each fixed set of parameters $(\alpha,\mathbf{p},\mathbf{q},r)$ when $T> 540\max \lcb \frac{m(\pv)}{\al},\frac{m(\qv)}{\alo}\rcb \log (NT)$. Here, $\alo = 1-\al$. 
For the Bernoulli HCB-mc as in Sec.~\ref{sec:Bernoulli HCB}, instead, there exists an algorithm $\phi_{\textsf{mc}}$ such that
%$(\al,\pv,\qv,r(\xv)) as in 
\begin{equation}
    R_T^{\mathsf{mc}}\leq 116\sqrt{ \lb  m(\pv)\al+m(\qv) \alo \rb \f {\log NT}{T}},
    \label{eq:upper bound mc}
\end{equation}
for each fixed $(\alpha,\mathbf{p},\mathbf{q},r)$ when $T> 540\max \lcb \frac{m(\pv)}{\al},\frac{m(\qv)}{\alo}\rcb \log (NT)$.
\end{theorem}

\begin{theorem}[Bernoulli HCB simple regret lower bound]\label{thm:lower_bound}
Consider a Bernoulli HCB as in Sec.~\ref{sec:Bernoulli HCB}, where the context node can be either manipulable or non-manipulable. 
Fix parameters $(\alpha, \mathbf{p},\mathbf{q})$ with $N\geq 4$. Fix an algorithm $\phi$.
Then there exists a reward function $r(\xv)$ for which it holds that
%, such that the regret satisfies
\begin{equation}
R_T\geq \frac{1}{127}\sqrt{\frac{\widetilde{m}(\al,\pv,\qv)}{T}},  
\end{equation}
for all  $T\geq \max\{m(\pv),m(\qv)\}$. Here,
\ea{
\widetilde{m}(\al,\pv,\qv)=\lcb \p{
\max\{m(\pv) \al^2,m(\qv) \alo^2  \} &   m(\pv) \geq \tau_1(\al), m(\qv) \geq \tau_0(\al) \\ 
m(\pv) \al^2 &  m(\pv) \geq \tau_1(\al), m(\qv) < \tau_0(\al)  \\
m(\qv) \alo^2 &   m(\pv) < \tau_1(\al), m(\qv) \geq \tau_0(\al)  \\
(1-\max\{q_{\max},1/2\})^2 &  m(\pv) < \tau_1(\al), m(\qv) < \tau_0(\al)
}\rnone
}
where  $\tau_0(\al)=\frac{3e}{\alo (3-e)}$, $\tau_1(\al)=\frac{3e}{\al (3-e)}$, and $q_{\max}=\max_{n \in [N]} q_n$.
 \end{theorem} 

A few important observations can be made. First, the metric $m$ does play an important role in determining the simple regret scaling of HCB. However, our lower bounds suggest that the complexity of an HCB is determined by the difficulty of the ``worst'' sub-problem, i.e., the larger one between $m(\mathbf{p})$ and $m(\mathbf{q})$, and not the difficulty associated with the marginal probability, $m(\mathbf{p}\al+\mathbf{q}\alo)$. This finding leads to a somewhat counter intuitive observation: an arm being highly stochastic in one context does not entail that its causal effect can be easily learned, so long as the same arm is highly biased in another context. The  driving insight behind this phenomenon, one that underpins the proof of our lower bounds, is that when $N$ is large, the reward function can effectively discriminate between different conditional arm value distributions, thus giving it the power to adversarially ``target'' specific context states and cause large regret. This effect seems to be a fundamental feature of the HCB model, and explains why it is not sufficient to only count the arms that are biased on-average when it comes to simple regret. 

Note that the upper bounds in \eqref{eq:upper bound nmc} and \eqref{eq:upper bound mc} in Thm.~\ref{thm:upper bound} only differ by a constant.  This shows that for the Bernoulli HCB model, the ability of manipulating the context directly does not introduce significant improvement in the performance.  
Note also that the upper and lower bounds in Thm.~\ref{thm:upper bound} and Thm.~\ref{thm:lower_bound} are nearly matching up to $\log$ terms as they differ only in a multiplicative factor that scales as $\Ocal(\sqrt{\log(NT)})$, while a discrepancy scaling as $\Ocal(\sqrt{\log(NT/m(\mathbf{p}))})$ is present in the bounds of \cite{lattimore2016causal}.  

Our final theorem generalizes the upper bound to HCB with a $K$-valued context. Note that the lower bound above holds (trivially) for the $K$-context setting as well. 

\begin{theorem}[$K$-context HCB]
\label{thm:upper bound K}
Consider a Bernoulli HCB with $K$-context as in Sec.~\ref{sec:Bernoulli HCB}.
If the context node is non-manipulable, then there exists an algorithm with regret
\begin{equation}
    R_T^{\mathsf{nmc}} \leq 27\sqrt{ \f{K(2K+1)\widetilde{m}\log NT}{T}}
\label{eq:upper bound nmc K}
\end{equation}
for each fixed set of parameters $\left(\bm{\alpha},\{\mathbf{p}^{(i)}\}_{i\in[0:K-1]},r\right)$ when the number of experiments $T$ satisfies
\begin{equation}\label{eq:Tlowerthm}
T> 600(7K+1)\max_{i\in[0:K-1]} \lcb \frac{m(\pv^{(i)})}{\al_i}\rcb \log (NT).
\end{equation}
If the context node is manipulable, there exists an algorithm such that 
\begin{equation}
    R_T^{\mathsf{mc}}\leq 7\sqrt{\f {K(7K+1)\widetilde{m} \log NT}{T}}.
    \label{eq:upper bound mc K}
\end{equation}
for each fixed $\left(\bm{\alpha},\{\mathbf{p}^{(i)}\}_{i\in[0:K-1]},r\right)$ when $T$ satisfies \eqref{eq:Tlowerthm}.
Here, 
$
    \widetilde{m} = \sum_{i\in[0:K-1]}\alpha_i m(\pv^{(i)}).
$
\end{theorem}

\section{Technical Preliminaries} 
\label{sec:Related Models}

Before delving into the proofs of the main results, let us introduce some further notation that will be useful in the development of the paper.

\subsection{Parallel Bandit Model}
\label{sec:Parallel Bandit Model}

The parallel bandit model studied in~\cite{lattimore2016causal}  is presented in Fig. \ref{fig:network_topologies_2}.
This model can be obtained from the Bernoulli HCB when $\mathbf{p}=\mathbf{q}$.
An algorithm is introduced in~\cite{lattimore2016causal} for the parallel bandit which attains the simple regret 
\ea{
R_T=\Ocal \left(\sqrt{m(\mathbf{p})\frac{\log (NT/m(\mathbf{p}))}{T}}\right).
}
%
%\vspace{5cm}
%with regret $R_T=O\left(\sqrt{\frac{m(\mathbf{p})\log (NT)}{T}}\right)$.
In particular, since the variables in $\Xv$ are mutually independent, it is easy to see that $P(Y=1|\DO(X_n=1)) = P(Y=1|X_n = 1)$ and $P(Y=1|\DO(X_n=0)) = P(Y=1|X_n=0)$. The agent is then able to utilize outcomes from pure observations to deduce the impact of interventions, which largely improves the efficiency of learning.

The authors of \cite{lattimore2016causal} also provide a lower bound for the parallel bandit for a fixed $\pv$ by showing that there exists a reward function $r(\xv)$ yielding a simple regret $R_T$ such that
%
%showing that with fixed $\mathbf{p}$, for any algorithm, there exists a reward function $r$ such that the regret satisfies $
$
R_T   =\Omega \left(\sqrt{\frac{m(\mathbf{p})}{T}} \right).  
$

\subsection{Formulation of Causal Bandit Algorithms}
\label{sec:Regret Upper Bounds}
In order to study the simple regret in \eqref{eq:pure exploration regret}, one has to specify how the agent sequentially determines the action at round $t \in [T]$, and identifies the optimal action after $T$ rounds of experiments, having observed the past values of the reward $Y$, the arms' realization $\Xv$, and that of the context $S$.
Let us more formally introduce the algorithm determining the agent's action and identifying the optimal intervention. 
Let $\Xv^{(t)}$, $Y^{(t)}$, and $S^{(t)}$ be the realizations of the variables in the $t^{\rm th}$ round for $t\in [T]$.
%
%Recall that $A^{(t)}$ is the action chosen at time $t$, let $\mathbf{A}_t = \left(A^{(1)},\ldots,A^{(t)}\right)$ be the history of actions up to time $t$. 
%
Let the history of the context, reward, arms, and actions up to time $t$ be denoted as   
$\Hv_t = \lcb S^{(\ell)},\Xv^{(\ell)},Y^{(\ell)},A^{(\ell)} \rcb_{\ell\in[t]}$. 
%the history of the all variables and actions in the causal bandit system up to time $t$.
%
%The agent conducts the experiments during the time horizon $[T]$ according to an algorithm $\phi$. 
%In general, 
%
An algorithm of the causal bandit problem (either $\mathsf{mc}$ or $\mathsf{nmc}$) $\phi$, is then defined as a collection of functions  
\ea{\lcb g^{(t)} \lb \Hv_{t-1}, W_A^{(t)} \rb \rcb _{t \in [T+1]},
\label{eq:algo}
}
where 
\ea{
g^{(t)} : \lb \Scal \times \Xcal^N \times \Ycal \times \Acal \rb^{t-1} \times [0,1] \goes \Acal,
\label{eq:algo2}
}
and $W_A^{(t)} \overset{\rm i.i.d.}{\sim}  \Ucal([0,1])$ are uniform random variables that characterize the randomness of the algorithm.
Here, for $t\in [T]$, $g^{(t)}$ is the function that determines the agent's action in the $t^{\rm th}$ experiment, while $g^{(T+1)}$ is the function that identifies an action $\widehat{A}_T$ as the optimal one based on the entire history after completing $T$ rounds of experiments.  
In \eqref{eq:algo2}, the set of actions $\Acal$ is specified as \eqref{eq:nmc} and \eqref{eq:mc} for the scenario of non-manipulable and manipulable context, respectively. 
Note that the algorithm $\phi$ does not depend on the distributions of the arms $\pv$, $\qv$, the distribution of the context, $\alpha$, or the reward function $r$, as these parameters are not provided to the agent and usually need to be learnt by the algorithm.

\subsection{Properties of KL Divergence}
\label{sec:Regret Lower Bounds}

When determining a lower bound on the regret, it is customary to consider the setting in which the arm distribution is fixed but the reward function is chosen by an adversary trying to maximize the simple regret. 
In particular, a lower bound on the simple regret can be obtained by designing a reward function for which finding the correct choice of the action $\Ah_T$ is particularly challenging.

In our converse proof, following \cite{lattimore2016causal}, we will crucially rely on the properties of the Kullback-Leibler (KL) divergence between two distributions $P$ and $Q$, defined as
\ea{
D_{\rm KL}(P_X,Q_X)=\sum_{ x \in \Xcal}  P_X(x)\log  \f {P_X(X)}{Q_X(x)}.
}
The conditional version of the KL divergence is defined as \ea{
D_{\rm KL} (P_{Y|X},Q_{Y|X})=\Ebb_{x \sim P_X} D_{\rm KL} (P_{Y|X=x},Q_{Y|X=x}).
}
The following two lemmas involving the KL divergence will be useful in the proof of lower bound.
\begin{lemma}%[Chain rule KL divergence]
[Properties of KL divergence \cite{cover1999elements,tsybakov2008introduction}]
Let $P(x,y)$ and $Q(x,y)$ be two distributions over the support $\Xcal\times\Ycal$ with $\Xcal = \Ycal =\{0,1\}$, then the KL divergence satisfies
\ea{
D_{\rm KL} (P_{XY},Q_{XY}) = D_{\rm KL} (P_{X},Q_{X})  +D_{\rm KL} (P_{Y|X},Q_{Y|X}).
}
Additionally $D_{\rm KL}(P_X,Q_X)\geq 0$ and
\ea{
D_{\rm KL}(P_X,Q_X) \leq \chi^2(P_X,Q_X)
%\f 1 {2 \log 2} \| P_X-Q_X\|_1^2.
}
where $\chi^2(P_X,Q_X) = \frac{(p_X-q_X)^2}{q_X(1-q_X)}$, with $p_X = P_X(X=1)$, $q_X = Q_X(X=1)$.
\end{lemma}

\begin{lemma}[Lemma 2.6 in~\cite{tsybakov2008introduction}]\label{lemma:subset}
For two probability distributions $P$, $Q$, on the same measurable space, 
\begin{equation}
    P(\Bcal)+Q(\Bcal^c)\geq\frac{1}{2}\exp\left(-D_{\rm KL}(P,Q)\right),
\end{equation}
for every measurable subset $\Bcal$ and its complement $\Bcal^c$. 
\end{lemma}

We are now ready to provide a proof for our main results, Thm.~\ref{thm:upper bound} and Thm.~\ref{thm:lower_bound}.
The proof of Thm. \ref{thm:upper bound K} is omitted as it can be proved in the same way as Thm. \ref{thm:upper bound}.

\section{Proof of Theorem \ref{thm:upper bound}
%Regret upper bound of the hierarchical causal bandit
}
\label{sec:upper}

In this section we provide the proof sketch for Thm. \ref{thm:upper bound} for the non-manipulable case.
The complete proof can be found in App.~\ref{app:proof_upper}.
We begin by introducing two simple lemmas to further characterize $m(\vv)$ in \eqref{def:m}, the proof of which can be found in Apps.~\ref{app:exists} and~\ref{app:small}.
\begin{lemma}
\label{lem:exists}
The quantity $m(\vv)$ in \eqref{def:m} is well defined for all probability vectors $\vv$.
\end{lemma}

\begin{lemma}
\label{lem:small}
Suppose the entries of $\vv \in (0,1)^N$ are such that $0\leq v_1\leq v_2\leq \ldots \leq v_N\leq 1/2$, then $v_n\leq 1/m(\vv)$ for all $n  \in [\lceil m(\mathbf{v})\rceil ]$.
\end{lemma}

%Hence, with 
%From Lemma \ref{lem:exists}, we see that the minimum in \eqref{def:m} always exists. 
%
Lem.~\ref{lem:small} more precisely characterizes our earlier statement that $m(\mathbf{v})$ is a measure of the problem complexity. 
%number of entries in $\vv$ that are highly biased. 
%
When all entries in $\vv$ are below $1/2$, then we have that the $\lceil m(\mathbf{v})\rceil$ smallest entries are at most $1/m(\vv)$.
%
%show that the value of $m(\mathbf{v})$ depicts the number of entries in a vector $\mathbf{v}$ that are highly biased, i.e., number of entries either close to zero or one. Since an event with probability close to zero is hard to observe, without the help of intervention, estimating values associated with this rare event will incur high variance. 
%
Intuitively, this set of outcomes is hard to observe and, consequently, it is hard to estimate the corresponding reward. 
In order to obtain a good estimate for such events, we need to  conduct a number of interventions proportional to $m(\vv)$.
For this reason, $m(\vv)$ serves as a metric of the problem complexity.    
%
%and use the corresponding observations to reduce the 
%
%to reduce the variance. Hence, the value of $m$ characterizes the number of random variables that are hard to estimate, thus serving as a metric of the problem complexity.    

Next, we introduce a lemma characterizing the effect of an intervention on an arm of the HCB model.
\begin{lemma}\label{lemma:intervention}
For the Bernoulli HCB model, the following statement about interventions are true:
\begin{align}
    \mathbb P(Y=1|\DO(X_i=x)) =  \alpha P(Y=1|X_i=x,S=1)  +\overline{\alpha}P(Y=1|X_i=x,S=0),
  \label{eq:do1}
\end{align}
for $i\in [N]$, $x\in\mathcal X$. In addition, we have
%\begin{align}
\eas{
\label{eq:do2}
    &P\left(Y=1|\DO(\varnothing), X_i=x,S=s\right) \nonumber\\
    & \quad \quad = P(Y=1|\DO(X_i=x),S=s)\\
%\end{align}
%and
%\begin{align}\label{eq:do3}
 %   &P\left(Y=1|\DO(\varnothing), X_i=x,S=s\right) \nonumber\\
    & \quad \quad = P(Y=1|\DO(S=s),X_i=x).
    \label{eq:do3}
    }
%\end{align}
with  $i\in [N]$, $x\in \mathcal X$, and $s \in \mathcal S$.
\end{lemma}
\begin{proof}
The proof of this lemma is provided in App.~\ref{app:intervention}.
Note that this lemma can be extended to the general HCB formulation. 
\end{proof}
From Lem.~\ref{lemma:intervention}, we see that intervening on the $i^{\rm th}$ arm, e.g., $\DO(X_i=1)$, has the effect of removing the effect of the stochasticity of that arm in the reward.
Also, having observed a certain realization of the context or the arm, we can interpret this outcome as arising from an intervention enforcing such realizations. 

This lemma is rather straightforward but it is the core of our  algorithm design. 
%
%$\phi_{\textsf{nmc}}$ and  $\phi_{\textsf{mc}}$ to achieve the upper bound. 
In particular, we see that the impact of the action $\DO(X_i=1)$ on the reward can be estimated from the estimates of $\alpha$, $P(Y=1|X_i=1,S=1)$, and $P(Y=1|X_i=1,S=0)$. 
In other words, we can estimate $P(Y=1|\DO(X_i=1))$ by utilizing samples collected through pure observation (corresponding to action $\DO(\varnothing )$) without actually conducting the action $\DO(X_i=1)$. 
This implies that $\DO(\varnothing )$ is very efficient, since the resulting observations can be used to 
%Hence, the learning efficiency can be greatly enhanced since we can 
simultaneously learn the impact of multiple actions.
Note that this observation is also valid for the 
%through pure observation, which is also the case for the 
parallel bandit model of \cite{lattimore2016causal}. 

Let us now introduce Alg.~\ref{alg:nmc} for the HCB-nmc model, which yields the simple regret \eqref{eq:upper bound nmc} in Thm.~\ref{thm:upper bound}.
The algorithm consists of two stages. In the first stage, we purely observe the system and estimate its parameters by averaging the random observations. 
In particular, we derive the estimates in lines 5--8 in Alg.~\ref{alg:nmc} for $\al,\pv,\qv$, %$\mu_{sc_1jc_2}:=P(Y=1|S=c_1,X_j=c_2)$, 
$\mu_{sljk}:=P(Y=1|S=l,X_j=k)$, and $\mu_{\DO(\varnothing )}=P(Y=1|\DO(\varnothing )$.
These estimated quantities are indicated with a hat in the notation in the following paragraphs.

From the estimates $\hat{\pv}$ and $\hat{\qv}$, we next estimate $m(\pv)$ and $m(\qv)$ and decide which estimates can be considered reliable. 
For instance, if the number of occurrences of the event $\{S=1,X_j=1\}$ surpasses a threshold determined by $m(\phv)$, then the estimator $\hat{\mu}_{s1j1}$ can be considered sufficiently accurate.
Otherwise, we will derive a better estimate for this value in the second stage of the algorithm. 

In the second stage, the algorithm focuses on the values of $\mu_{sljk}$ that have not been estimated accurately enough in the first stage. For each of these values, we conduct intervention $\DO(X_j=k)$ to increase the chance of observing the corresponding event, thus resulting in a better estimate.  
From Lem.~\ref{lem:small} we see that, with high probability, the number of such values is proportional to $m(\pv)$ or $m(\qv)$. %multiplied by a constant. 
Hence, we are able to allocate a sufficient number of experiments to intervene each required variable such that the upper bound in \eqref{eq:upper bound nmc} of Thm.~\ref{thm:upper bound} can be satisfied.  
So far, we have sketched the proof of the upper bound for the non-manipulable case through constructing an algorithm satisfying the performance requirement. 
In Alg.~\ref{alg:nmc}, in lines 13 and 14 we define
\ea{
\Bcal(\vv,z)= \lcb j \in [N], v_j< \f 1 z \rcb 
\label{eq:bcal}
}
for $\vv \in (0,1)^N$, $z\in \mathbb R_+$.  That is, Eq.~\eqref{eq:bcal} indicates the entries of $\vv$ that are highly biased towards zero according to the threshold $1/z$.

After the estimates of each entry $\mu_{sljk}$ are sufficiently precise,  the estimate of $\muh_{\DO(X_j=1)}$ is obtained from $\alh$, $\muh_{s1j1}$, and $\muh_{s0j1}$.  See Alg.~\ref{alg:nmc} lines 22-24.
Finally, the algorithm outputs $\hat{A}_T \in \Acal^{\rm nmc}$, the action that corresponds to the largest estimated reward.

\begin{algorithm}
%\SetAlgoLined
  \SetKwFunction{FMain}{Refine}
  \SetKwProg{Fn}{Function}{:}{}
  \Fn{\FMain{$\Bcal$,$s$,$x$,$\tau$,$d$}}{
  $i=1$\;
  
   \For {$j \in \Bcal$}{
\For { $t \in [\tau+(i-1)d/|\mathcal B|: \tau + id/|\Bcal|]$}{
choose $\DO(X_j=x)$ \;

}
$i_j=i$, $i=i+1$\;
}
\For {$j \in \Bcal$}{
         $c_j=\sum_{t \in [\tau+(i_j-1)d/|\mathcal B|:\tau+i_j d/|\mathcal B|]} \onev_{\{S^{(t)}=s \}}$ \;
         
         $f_j=\sum_{t \in [\tau+(i_j-1)d/|\mathcal B|:\tau+i_j d/|\mathcal B|]} \onev_{\{S^{(t)}=s,Y^{(t)}=1 \}}$ \;
        
        $u_j = f_j/c_j$\;
        }
        
  }
  %}

\KwResult{$\uv$}

\vspace{0.25cm}
\tcp{Stage 1} 
\vspace{0.25cm}

Fix $T'=T/5$

\For{ $t \in [T']$}{
choose $\DO(\varnothing )$ \;
}

estimate $\al$ as $\alh=\f {1} {T'} \sum_{t \in [T']} S^{(t)}$\;

estimate $p_i$ as $\ph_i=\f 1 {T' \alh} \sum_{t \in [T']} X_i^{(t)}S_i^{(t)}$\;

estimate $q_i$ as $\qh_i=\f 1 {T' (1-\alh)} \sum_{t \in [T']} X_i^{(t)}(1-S_i^{(t)})$\;

estimate $\mu_{\DO(\varnothing )}$ as  $\muh_{\DO(\varnothing )} =\f {1} {T'} \sum_{t \in [T']} Y^{(t)}$\;

estimate $\mu_{s1j1}$ %=P(Y=1|X_j=1,S=1)$ 
as \\
$\muh_{s1j1}=\f 1 {T' \ph_j \alh} \sum_{t \in [T']}Y^{(t)}X_j^{(t)}S^{(t)}$\;

estimate $\mu_{s1j0}$  %=P(Y=1|X_j=1,S=1)$
as \\ 
$\muh_{s1j0}=\f 1 {T' (1-\ph_j) \alh} \sum_{t \in [T']}Y^{(t)}(1-X_j^{(t)})S^{(t)}$\;

estimate $\mu_{s0j1}$  %=P(Y=1|X_j=1,S=1)$
as  \\
$\muh_{s0j1}=\f 1 {T' \ph_j (1-\alh)} \sum_{t \in [T']}Y^{(t)}X_j^{(t)}(1-S^{(t)})$\;

estimate $\mu_{s0j0}$  %=P(Y=1|X_j=1,S=1)$
as \\
 $\muh_{s0j0}=\f 1 {T' (1-\ph_j) (1-\alh)} \sum_{t \in [T']}Y^{(t)}(1-X_j^{(t)})(1-S^{(t)})$\;

\vspace{0.25cm}
\tcp{Stage 2} 
\vspace{0.25cm}

let  $\hat{\mathcal B}_{11}=\Bcal(\phv,m(\phv))$,  $\hat{\mathcal B}_{10}= \Bcal(1-\phv,m(\phv))$ \;

let  $\hat{\mathcal B}_{01}=\Bcal(\qhv,m(\qhv))$,  $\hat{\mathcal B}_{00}= \Bcal(1-\qhv,m(\qhv))$ \;

$\uv_{11}$ = \FMain($ \hat{\mathcal B}_{11}$, $1$, 1, $T'$, $T'$) \;

$\uv_{10}$ = \FMain($ \hat{\mathcal B}_{10}$, 1, $0$, $2T'$, $T'$) \;

$\uv_{01}$ = \FMain($ \hat{\mathcal B}_{01}$, 0, $1$, $3T'$, $T'$) \;

$\uv_{10}$ = \FMain($ \hat{\mathcal B}_{10}$, 0, $0$, $4T'$, $T'$) \;

\For{ $l\in\Scal,k\in\Xcal,$}{
update $\muh_{sl jk}$  with $\uv_{lk}$ for $j\in\hat{\mathcal B}_{lk}$ \;
}
\For{ $j \in [N]$}{
 $\muh_{\DO(X_j=1)}=\alh \muh_{s1j1}+(1-\alh)\muh_{s0j1}$ \;
 
 $\muh_{\DO(X_j=0)}=\alh \muh_{s1j0}+(1-\alh)\muh_{s0j0}$ \;
}

output 
$\hat{A}_T = \argmax_{a\in \mathcal A^{\textsf{nmc}}}   \muh_{a}$
%\ \{\muh_{s1jv},\muh_{s0jv} \}
%}

\caption{HCB-nmc algorthm used in Thm.~\ref{thm:upper bound}.}
\label{alg:nmc}
\end{algorithm}

For the scenario with manipulable context, we introduce Alg.~\ref{alg:mc} that achieves the upper bound \eqref{eq:upper bound mc} of Thm.~\ref{thm:upper bound} in App.~\ref{app:proof_upper}. 
This algorithm differs from  Alg.~\ref{alg:nmc} in that a portion of the interventions is dedicated to estimating $\pv$, $\qv$, and values of $P(Y=1|S=l,X_j=k)$ by manipulating the context. 
Further discussion is relegated to the appendix.

\section{Proof of Theorem \ref{thm:lower_bound}}
\label{sec:lower}

In this section we prove the lower bound for the hierarchical model in  Thm. \ref{thm:lower_bound}.
We begin by introducing the following lemmas: their proofs are provided in Apps.~\ref{app:kl} and \ref{app:1}.

\begin{lemma}\label{lemma:kl}
Fix an algorithm $\phi$. Fix system parameters $(\alpha,\mathbf{p},\mathbf{q})$. Let $r_0$, $r_i$ be two reward functions. Let $P_j$ be the probability distribution of $\Hv_T$ when the system parameters are $(\alpha,\mathbf{p},\mathbf{q},r_j)$ and the algorithm is $\phi$, for $j\in\{0,i\}$. The KL divergence between the probability distributions $P_0$ and $P_i$ can be expanded as 
\begin{align}
    &D_{\rm KL}(P_0(\Hv_T),P_i(\Hv_T)) \nonumber\\
    = &\sum_{t\in[T]} \sum_{\mathbf{x}^{(t)}\in\Xcal^{N}} P_0\left(\mathbf{x}^{(t)}\right)\Bigg(r_0\left(\mathbf{x}^{(t)}\right)\log\frac{r_0\left(\mathbf{x}^{(t)}\right)}{r_i\left(\mathbf{x}^{(t)}\right)}+\left(1-r_0\left(\mathbf{x}^{(t)}\right)\right)\log\frac{1-r_0\left(\mathbf{x}^{(t)}\right)}{1-r_i\left(\mathbf{x}^{(t)}\right)}\Bigg).
\end{align}
\end{lemma}

\begin{lemma}\label{lem:1}
Fix $\alpha\in (0,1)$, and $\mathbf{p},\mathbf{q}\in (0,1)^{N}$ with $p_1\leq\ldots\leq p_{N}\leq\frac{1}{2}$ and $m(\pv)>2$.
Let $\pi:[\lceil m(\pv)\rceil]\mapsto [\lceil m(\pv)\rceil]$ be the permutation that sorts $(q_1,\ldots,q_{\lceil m(\pv)\rceil })$ in increasing order, i.e., $q_{\pi(1)}\leq \ldots\leq q_{\pi(\lceil m(\pv)\rceil )}$.
Define 
\begin{equation}\label{def:X_star}
    \Xcal_i^* = \left\{\mathbf{x}: x_i = 1, x_\ell = 0, \ell\in [\lceil m(\pv) \rceil]\backslash\{i\}\right\}
\end{equation}
for each $i\in [\lceil m(\pv) \rceil]$, and define 
$
\mathcal I = \left\{\pi(j), 1\leq j\leq \left\lfloor\frac{m(\pv) }{2}\right\rfloor\right\}.
$
For each $i\in [\lceil m(\pv) \rceil]$, we have 
\begin{equation}
    \frac{\alpha}{e}\leq P\lb \Xv\in \Xcal_{i}^*|\DO(X_{i}=1) \rb \leq 1,
\end{equation}
and for each $i\in \mathcal I$, we have 
$
    P\lb \Xv\in \Xcal_{i}^*|a\rb \leq \frac{1}{m(\pv) },
$
for any action $a\neq \DO(X_i=1)$. 
\end{lemma}

Next we will use the definition of the set $\Xcal_i^*$ in Lem.~\ref{lem:1} to construct the reward function yielding the desired lower bound as follows. 
Take $m(\pv) \geq m(\qv)$ and fix a constant-valued reward function $r_0$.
For each $i$, we introduce another reward function $r_i$, which is greater than $r_0$ by a positive value $\varepsilon$ if $\Xv$ falls into the set $\Xcal_i^*$, while $r_i$ equals $r_0$ otherwise. 
Hence, under reward function $r_i$, the optimal action is the one that maximizes the probability of $\Xv$ reaching the target set $\mathcal X_i^*$.
With Lem.~\ref{lem:1}, we see that the optimal action corresponding to reward function $r_i$ is $\DO(X_i=1)$. Under this optimal action, the probability that $\Xv\in \Xcal_i^*$ is at least $\alpha/e$, while the probability of reaching this set under any other action is at most $1/m(\pv)$. 
Hence, the regret of identifying the wrong arm when the reward function is $r_i$ is at least $\varepsilon\left(\frac{\alpha}{e}-\frac{1}{m(\pv) }\right)$.  

Next, we introduce the probability distribution of the entire history $\Hv_T$ corresponding to $r_i$, denoted by $P_i$, and that corresponding to $r_0$, denoted by $P_0$. We show that the KL divergence between these two distributions is upper bounded by a constant.

With this upper bound on $D_{\rm KL}(P_0,P_i)$, we can derive a lower bound for $P_0(\mathcal B)+P_i(\mathcal B^c)$ for any measurable event $\mathcal B$ according to Lem.~\ref{lemma:subset}.
In particular, let $\mathcal B_i$ be the event that the algorithm identifies $\DO(X_i=1)$ as the optimal intervention. It can be shown that there exists $i'\in [N]$ such that $P_{i'}\left(\mathcal B_{i'}^c\right)$ is lower bounded by a positive constant value. In other words, there exists a reward function $r_{i'}$ such that the probability of the agent identifying the incorrect action is at least a positive constant.
Since we have already shown a lower bound for the regret assuming a wrong action is identified by the agent, 
we have completed the proof that there exists a reward function $r_{i'}$ such that the regret is at least $\varepsilon\left(\frac{\alpha}{e}-\frac{1}{m(\pv) }\right)$ multiplied by a constant.

\section{Conclusion}
\label{sec:conclusion}
In this paper, we proposed and studied the hierarchical causal bandit (HCB) problem as a step towards understanding general causal bandits with cross-arm dependencies and interactions. Through upper and lower bounds on the simple regret, our results suggest that the critical determinant of regret is the number of biased arms under the worst context, and not the number of biased arms under the marginal arm distribution. In a way, our results can be interpreted as saying that overall hardness is the average (upper bound) or maximum (lower bound) hardness of the sub-problems, and not the hardness of the ``average'' sub-problem. 

The above observations also reveal a practical fragility of the parallel causal bandit model in the face of non-independent arms. That is,  the simple regret predictions offered by the parallel bandit model can be highly inaccurate in an environment where cross-dependencies are present. One might argue that the comparison is not fair since the parallel bandit model is intended for independent arms (i.e., with a single context). Nevertheless, it is still interesting to see a large  discrepancy as soon as we incorporate two possible contexts, a seemingly moderate number. 

There are several interesting open questions: the dependencies on $m$ in our upper and lower bounds do not fully agree, and this is in part due to the fact that the upper bound is still $\sqrt{\log(T)}$ factor greater than the lower bound. Closing this gap would help us sharpen our understanding of what drives simple regret in hierarchical bandits. Similarly, the bounds for the case of a $K$-context causal bandit with $K>2$ are still quite loose.  On the modeling end, there clearly remains a sizable complexity gap between HCB and a full-blown structural causal model. We are hopeful that the HCB model can be extended and augmented to gradually close this gap, by, for instance, considering multiple contextual nodes, or multi-layer hierarchies (the current model has only 1 layer).

\bibliographystyle{acm}
\bibliography{causal}

\newpage
\onecolumn
\section*{Appendix}

\appendix
Recall that
\begin{equation}
    \mu_a = \mathbb E[Y|A=a]
\end{equation}
is defined as the expected reward when the agent chooses action $a$. Note also that $a^* = \max_{a\in\mathcal A}\mu_a$, and $\mu^* = \max\mu_a = \mu_{a^*}$.  For $x\in \mathbb R$, $\delta\in \mathbb R_+$, let $[x\pm \delta] = [x-\delta, x+\delta]$.  
For $i,j\in \mathbb Z_+$, $i<j$, let $[i:j] =\{i,\ldots,j\}$, and when $i=1$, we introduce the shorthand $[j]= [1:j]$. 
Recall that 
\begin{equation}\label{eq:musljk}
\mu_{sljk} = P(Y=1|X_j=k,S=l),
\end{equation}
for $j\in [N]$, and $l\in\Scal$, $k\in\Xcal$.
For an event $B$, recall that $\mathbf{1}_B$ is the indicator function that returns 1 if $B$ is true and 0 otherwise. 

We will make repeated use of the following error term: for $x,y,t\in \Rbb^+$, let 
\begin{equation}
\varepsilon_{x,y,t} =\sqrt{\frac{x\log (yt)}{t}}.
\end{equation}
Note that $\varepsilon $ is  increasing in $x$ and $y$ and  decreasing in $t$. While it is important for our bounds to keep track of dependencies of $\varepsilon_{x,y,t}$ on its input values $x,y$ and $t$, it is worth noting that, for obtaining the main insights from the proof, it suffices to view $\varepsilon_{x,y,t}$ as a suitably small constant. When appropriate, we shall highlight in context the exact nature of dependencies. 
We will also use the shorthand $m_1 = m(\mathbf{p})$, and $m_0 = m(\mathbf{q})$. 
Finally, since the underlying causal network is time-homogeneous, we may sometimes drop the time index in the variables when there is no ambiguity.

\section{Proof of Theorem \ref{thm:upper bound}}\label{app:proof_upper}

Now we formally prove Theorem \ref{thm:upper bound}: first, we consider the case that the context node cannot be intervened by the agent in Sec. \ref{sec:pf:nmc}, then the case in which it can be manipulated in Sec. \ref{sec:Manipulable context node}. 

\subsection{Non-manipulable context node}\label{sec:pf:nmc}
\begin{proof}
With Lem.~\ref{lemma:intervention}, the expected reward conditioned on an action $a$, $\mu_a$, can be expressed as a function of $\alpha$ and the values $\mu_{sljk}$, $j\in[N]$, $l\in \mathcal S$, $k\in\mathcal X$. Recall the definition in Eq.~\eqref{eq:musljk}, $\mu_{sljk}$ is the expected reward conditioned on the event $\{S=l$, $X_j=k\}$, which may be estimated through pure observation. Hence, we would like to design an algorithm that estimates $\mu_a$, $a\in \mathcal A$, through estimating $\mu_{sljk}$ and $\alpha$. In particular, we introduce Alg.~\ref{alg:nmc}, a two-stage algorithm that efficiently estimates $\mu_a$ for each action $a$, satisfying the upper bound given in Thm.~\ref{thm:upper bound}. We will show that under this algorithm, for many actions $a$, $\mu_a$ can be accurately estimated simultaneously with pure observation. This significantly improves the efficiency of the algorithm compared to the alternative, where the agent experiments with every action $a$ to estimate each $\mu_a$ separately.

\noindent
\underline{\bf Stage 1:}
In the first stage, we conduct a total of $T' = T/5$ experiments when the agent simply observes the system, or equivalently, conducts the empty intervention $\DO(\varnothing)$.
In this stage, the agent estimates the parameter $\alpha$, the vectors $\mathbf{p}$, $\mathbf{q}$, their $m$ values, $m_1 = m(\pv)$, and $m_0 = m(\qv)$, and values $\mu_{sljk}$. Their corresponding  estimators are denoted by $\hat{\alpha}$, $\hat{\pv}$, $\hat{\qv}$, $\hat{m}_1 = m(\hat{\pv})$, $\hat{m}_2 = m(\hat{\qv})$, and $\hat{\mu}_{sljk}$. 
Since the estimator $\hat{\mu}_{sljk}$ can only be calculated when the event $\{S=l,X_j=k\}$ occurs, this estimator can only be accurate if the corresponding event takes place for a sufficient number of times.  In particular, we will introduce a threshold determined by $\hat{m}_1$ (for $k=1$) and $\hat{m}_0$ (for $k=0$). If the number of times that the event occurs surpasses this threshold, the accuracy of $\hat{\mu}_{sljk}$ is good enough. In this case, we accept this estimation, and refer to $j$ as an ``easy" entry to estimate. 
Otherwise, we will not accept this estimation. The entry $j$ is then a ``hard“ entry to estimate and we will refine this estimate in Stage 2. In addition, we point out that the expected reward under the empty intervention, $\mu_{\DO(\varnothing)}$, can also be estimated in this stage as we conduct $T'$ steps of pure observation. 

Now we focus on the estimators of $\alpha$, $\pv=(p_1,\ldots,p_N)$, and $\mu_{s1j1}$, $j\in [N]$, while the analysis of $\qv$ and the other values of $\mu_{sljk}$ follows analogously. 
Recall the following estimators defined in Alg.~\ref{alg:nmc}:
\begin{align}
\hat{\alpha}=& \frac{1}{T'}\sum_{t\in[T']}S^{(t)}, \quad \mbox{(line 5 of Alg.~\ref{alg:nmc}),} \label{def:alphahat}\\
\hat{p}_i =& \frac{\sum_{t\in [T']}X_j^{(t)}S^{(t)}}{\sum_{t\in [T']}S^{(t)}}, \quad \mbox{(line 6)},\label{def:pihat} \\
\muh_{\DO(\varnothing )} =& \f {1} {T'} \sum_{t \in [T']} Y^{(t)}, 
\quad \mbox{(line 8)}, \label{def:muhatobserve}\\
\hat{\mu}_{s1j1} =& \frac{\sum_{t\in [T']}Y^{(t)}X_j^{(t)}S^{(t)}}{\sum_{t\in [T']}X_j^{(t)}S^{(t)}}, \quad \mbox{(line 9)}. \label{def:muhats1j1}
\end{align}  

Here, if the denominator of an estimator defined above happens to be zero, we simply set the estimator to be zero.
We introduce the following two lemmas that will be useful for subsequent proofs. In Lem.~\ref{lem:additional1}, we show that the estimators $\hat{\alpha}$, $\hat{\pv}$, $\muh_{\DO(\varnothing)}$, and $\hat{\mu}_{s1j1}$ all concentrate in an appropriate sense, and in Lem.~\ref{lem:additional2}, we show that the estimator $\hat{m}_1$ is accurate up to a constant factor with high probability, which can be derived similar to Lem.~8 in~\cite{lattimore2016causal}.
We define the following two events that will be useful in these two lemmas. 
Let $E_{\pv}$ be the event that every $\hat{p}_i$ is accurate up to a sufficiently small error: 
\begin{equation}
    E_{\pv} = \{\hat{p}_i\in \left[\left(1\pm\varepsilon_{{27}/{(\alpha p_i)},2N,T'}\right)p_i\right],\ \forall i\in [N]\}. 
\end{equation}
Similarly, define an event $E_{\overline{\pv}}$ as follows,
\begin{equation}
    E_{\overline{\pv}} = \{1-\hat{p}_i\in \left[\left(1\pm\varepsilon_{{27}/{(\alpha \overline{p}_i)},2N,T'}\right)\overline{p}_i\right],\ \forall i\in [N]\}. 
\end{equation}
 The proofs of these two lemmas are given in Apps.~\ref{app:additional1} and \ref{app:additional2}, respectively.
\begin{lemma}\label{lem:additional1}
The following concentration inequalities hold. 
\begin{enumerate}
    \item For estimator $\hat{\alpha}$, 
    \begin{equation}\label{eq:alpha}
     P\left(\left|\hat{\alpha}-\alpha\right|\leq 
     \varepsilon_{3\alpha,2,T'}\right)\geq 1-\frac{1}{T'}.
\end{equation}
\item For sufficiently large $T'$ such that $T'> \frac{27\log (2NT')}{\alpha}$, the events $E_{\pv}$ and $E_{\overline{\pv}}$ satisfy 
\begin{equation}\label{eq:pj1}
    P\left(E_{\pv}\right) \geq 1-\frac{2}{T'}, \quad P\left(E_{\overline{\pv}}\right) \geq 1-\frac{2}{T'}.  
\end{equation}

\item For the estimator $\muh_{\DO(\varnothing)}$, 
\begin{equation}\label{eq:observe}
     P\left(\left|\hat{\mu}_{\DO(\varnothing)}-\mu_{\DO(\varnothing)}\right|\geq \varepsilon_{3,2,T'}\right)\leq \frac{1}{T'}.
\end{equation}
\item For the estimator $\muh_{s1j1}$, 
\begin{equation}\label{eq:mus1j1}
    P\left(\hat{\mu}_{s1j1}\in\left[  \mu_{s1j1} \pm\varepsilon_{27/(\alpha p_j),2N,T'}\right]\right)\geq 1-\frac{2}{NT'}, \quad j\in[N]. 
\end{equation}
\end{enumerate}
\end{lemma}
\begin{lemma}\label{lem:additional2}
If events $E_{\pv}$ and $E_{\overline{\pv}}$ are simultaneously true, then for all sufficiently large $T'$ such that $T'> \f {108m_1\log(2NT')}{\alpha}$, $\hat{m}_1$ satisfies that 
\begin{equation}\label{eq:m1bound0}
     \frac{2m_1}{3}\leq\hat{m}_1\leq 2m_1.
\end{equation}
\end{lemma}

Note that with Eq.~\eqref{eq:pj1} and the union bound, the probability that both $E_{\pv}$ and $E_{\overline{\pv}}$ are true is at least $1-4/T'$.
Combining with Lem.~\ref{lem:additional2}, we have 
\begin{equation}\label{eq:m1bound1}
     P\left(\frac{2m_1}{3}\leq\hat{m}_1\leq 2m_1\right)\geq 1-\frac{4}{T'}.
\end{equation}
We now use the above two lemmas to bound the accuracy of $\muh_{s1j1}$, the estimator for the expected reward when both the context $S$ and the variable $X_j$ is $1$.  In particular, we show that the estimator $\muh_{s1j1}$ is sufficiently accurate when $\hat{p}_j\geq 1/\hat{m}_1$. 
Recall that 
$$ \hat{\mathcal B}_{11} = \left\{j\in [N]:\hat{p}_j< \frac{1}{\hat{m}_1}\right\}$$
which is first defined in line 13 of Alg.~\ref{alg:nmc}.
This is the set of indices of entries in $\hat{\pv}$ that are highly biased towards zero.
Fix $j$ to be an arbitrary index in  $[N]\setminus \hat{\mathcal B}_{11}$. Consequently, we have that $\hat{p}_j\geq 1/\hat{m}_1$. We first show that if the estimated entries $\hat{p}_j$ are not highly biased to zero, then neither will be the underlying true values $p_j$. In particular, since $\hat{p}_j$ and $\hat{m}_1$  are close to $p_j$ and $m_1$, respectively, we would expect that $p_j$ should be on the order of $1/m_1$ or greater for $j \in [N]\setminus \hat{\mathcal B}_{11}$. The following lemma makes this notion precise:

\begin{lemma} 
\label{lem:pjtom1}
We have that 
\begin{equation}\label{eq:pj2}
P
\left(p_j\geq\frac{1}{4m_1},\quad \forall j\in [N]\setminus\hat{\mathcal B}_{11} \right)  \geq 1-\frac{4}{T'}. 
\end{equation}
\end{lemma}
The proof of this lemma can be found in App.~\ref{app:pjtom1}.

By substituting the bounds on $p_j$ in Lem.~\ref{lem:pjtom1} into Eq.~\eqref{eq:mus1j1} and applying the union bound, we obtain the following characterization of  $\hat{\mu}_{s1j1}$: 
\begin{equation}\label{eq:B}
    P\left(\hat{\mu}_{s1j1}\in \left[ \mu_{s1j1} \pm\varepsilon_{108m_1/\alpha,2N,T'}\right],\ \forall j\in[N]\setminus\hat{\mathcal B}_{11}\right)\geq 1-\frac{6}{T'},
\end{equation}
for all large $T'$ where $T' > \frac{108m_1\log (2NT')}{\alpha}$.
In other words, we have shown that the set $[N]\setminus \hat{\mathcal B}_{11}$ includes the ``easy'' entries in $\mu_{s1j1}$, which can be well estimated via pure observation with the error bound in Eq.~\eqref{eq:B}. 
We omit the discussions of $\muh_{sljk}$ for $(l,k)=(1,0), (0,1), (0,0)$ as they follow the same argument as above by replacing $\alpha$, $\pv$, and $m_1$, with $\overline{\alpha}$, $\qv$, and $m_0$, respectively.

\noindent
\underline{\bf Stage 2:} 
Since the ``easy'' entries have already been well estimated by pure observation in the previous stage, in this stage, we focus on the ``hard'' entries $j\in \hat{\Bcal}_{11}$. 
In order to obtain sufficiently accurate estimates, we conduct a total of $4T'$ non-empty interventions targeted at these ``hard'' entries, where $T'=T/5$.

To this end, we will estimate $ \mu_{s1j1} $ for $j\in\hat{\Bcal}_{11}$ using $T'$ experiments.  
For each $j\in\hat{\mathcal B}_{11}$, we set the action to be $\DO(X_j=1)$ over ${T'}/{|\hat{\mathcal B}_{11}|}$ experiments.
Recall the definition of $\hat{m}_1$, we have 
$
\hat{\Bcal}_{11}\leq \hat{m}_1,
$
i.e., the number of ``hard'' entries in $\mu_{s1j1}$ is at most $\hat{m}_1$. It follows that the number of active interventions targeted at each ``hard'' entry is at least $T'/\hat{m}_1$. 
Let $\mathcal T_{j,1}\subseteq [T]$ be the subset of the time steps when the agent chooses action $\DO(X_j=1)$, i.e,  $A^{(t)}=\DO(X_j=1)$ for $t\in \mathcal T_{j,1}$. As indicated in line 15 of Alg.~\ref{alg:nmc}, 
for each $j\in \hat{\Bcal}_{11}$, we introduce a refined estimator
\begin{equation}\label{eq:stage2-1}
    \hat{\mu}_{s1j1}^{(r)} = \frac{\sum_{t\in\mathcal T_{j,1}} Y^{(t)}S^{(t)}}{\sum_{t\in\mathcal T_{j,1}} S^{(t)}}. 
\end{equation}

Similar to the analysis in Stage 1, we combine Chernoff bounds and the bound on $\hat{m}_1$ in Eq.~\eqref{eq:m1bound1} to derive the concentration inequality. 
By utilizing the union bound, we have 
\begin{equation}\label{eq:stage2-2}
    P\left( \hat{\mu}_{s1j1}^{(r)}\in \left[ \mu_{s1j1} +\varepsilon_{108m_1/\alpha,2N,T'}\right],\ \forall j\in\hat{\mathcal B}_{11}\right)\geq 1-\frac{4}{T'}-\frac{4m_1}{T'},
\end{equation}
when $T'$ satisfies 
\begin{equation}\label{eq:Tlower}
T'> \frac{108m_1\log (2NT')}{\alpha}.
\end{equation}
Together with the bound on $\hat{\alpha}$ in Eq.~\eqref{eq:alpha},  by union bound we have the following concentration inequality on the value of $\muh_{s1j1}^{(r)}\hat{\alpha}$:
\begin{equation}\label{eq:stage2-3}
    P\left(\left|\hat{\mu}^{(r)}_{s1j1}\hat{\alpha}- \mu_{s1j1} \alpha\right|\leq \varepsilon_{165\alpha m_1,2N,T'},\ \forall j\in\hat{\mathcal B}_{11}\right)\geq 1-\frac{5}{T'}-\frac{4m_1}{T'},
\end{equation}
when $T'$ satisfies Eq.~\eqref{eq:Tlower}. 
Combining the above bound for the ``hard'' entries with that for the ``easy'' entries in Eq.~\eqref{eq:B}, we have  
\begin{equation}\label{eq:hard11}
    P\left(\left|\hat{\mu}'_{s1j1}\hat{\alpha}- \mu_{s1j1} \alpha\right|\leq \varepsilon_{165\alpha m_1,2N,T'},\ \forall j\in [N]\right)\geq 1-\frac{7}{T'}-\frac{4m_1}{T'},
\end{equation}
where 
\begin{equation}
    \hat{\mu}'_{s1j1} = \muh_{s1j1} \mbox{ for }j \in [N]\setminus \hat{\Bcal}_{11}, \mbox{ and }   \hat{\mu}'_{s1j1} = \muh_{s1j1}^{(r)}  \mbox{ for} j\in \hat{\Bcal}_{11}.
\end{equation}
Other terms $\hat{\mu}'_{sljk}$ can be defined in the same way.
Analogously, we can estimate the other ``hard'' entries by conducting interventions targeting them. For example, we have 
\begin{equation}\label{eq:hard01}
    P\left(\left|\hat{\mu}_{s0j1}'\overline{\hat{\alpha}}- \mu_{s0j1} \overline{\alpha}\right|\leq \varepsilon_{165\overline{\alpha} m_0,2N,T'},\ \forall j\in [N]\right)\geq 1-\frac{7}{T'}-\frac{4m_0}{T'}.
\end{equation}

With Lem.~\ref{lemma:intervention}, we can obtain an estimator for $\mu_{\DO(X_j=1)}$ as
\begin{equation}
 \hat{\mu}_{\DO(X_j=1)} = \muh_{s1j1}'\hat{\alpha}+\muh_{s0j1}'\overline{\hat{\alpha}}.
\end{equation}
Similarly, we can derive an estimator for $\mu_a$, denoted by $\hat{\mu}_a$, for each $a\in \mathcal A^{\textsf{nmc}}$. 
Therefore, 
from Eqs.~\eqref{eq:hard11} and \eqref{eq:hard01}, and the union bound, we have
\begin{equation}\label{eq:mua}
    P\left(|\hat{\mu}_a -\mu_a|\leq \delta,\ \forall a\in\Acal\right)\geq 1-\frac{140+40(m_1+m_0)}{T},
\end{equation}
where 
\begin{equation}
\delta = 5\sqrt{33}\left(\sqrt{m_1\alpha}+\sqrt{m_0\alb}\right)\sqrt{\frac{\log (NT)}{T}},
\end{equation}
when $T$ is sufficiently large such that $T> 540\max\left\{\frac{m_1}{\alpha},\frac{m_0}{\alb}\right\}\log (NT)$.
Here, we substituted $T' = T/5$.

Now we are ready to obtain the upper bound on the simple regret for Alg.~\ref{alg:nmc}.
Recall that $\hat{A}_T$ is the action chosen by the agent, while $a^*$ is the underlying optimal action. As a result, we have that $\muh_{\hat{A}_T}\geq \muh_a$, and $\mu^* = \mu_{a^*}\geq \mu_a$, for all $a\in \mathcal A^{\textsf{nmc}}$.  In particular, according to  Eq.~\eqref{eq:mua}, the event of all $\muh_a$ being accurate up to an error of $\delta$ occurs with probability at least $1-\frac{140+40(m_1+m_0)}{T}$. If this event occurs, we have 
\begin{equation}
\mu_{\hat{A}_T}\geq \muh_{\hat{A}_T}-\delta\geq \muh_{a^*}-\delta\geq \mu_{a^*}-2\delta = \mu^* -2\delta. 
\end{equation}
In other words, if this event occurs, the regret is at most $2\delta$. On the other hand, if this event does not occur, the regret is at most 1. Therefore, we have
\begin{align}
    R_T^{\textsf{nmc}} 
    \leq & 2\delta +\frac{140+40(m_1+m_0)}{T}\nonumber\\
    = & 10\sqrt{33}\left(\sqrt{m_1\alpha}+\sqrt{m_0\alb}\right)\sqrt{\frac{\log (NT)}{T}}+\frac{140+40(m_1+m_0)}{T}\nonumber\\
    \leq & 61\left(\sqrt{m_1\alpha}+\sqrt{m_0\alb}\right)\sqrt{\frac{\log (NT)}{T}}\nonumber\\
    \leq & 
    122 \sqrt{\frac{(m_1\alpha+m_0\alb)\log(NT)}{T}}, 
    \label{eq:Rtnmc}
\end{align}
where the last step follows from the Cauchy-Schwartz inequality. The derivation in \eqref{eq:Rtnmc} proves  Eq.~\eqref{eq:upper bound nmc} in Thm.~\ref{thm:upper bound}.
\end{proof}

\subsection{Manipulable context node}
\label{sec:Manipulable context node}
Now we study the scenario when the context node can be intervened by the agent.

\setcounter{AlgoLine}{0}
\begin{algorithm}
\SetAlgoLined
 
\KwResult{Choice  $\Ah_T=a$ with $a \in \Acal^{\rm mc}$ for the HCB-mc model}

Fix $T'=T/15$, $T'' = T/5$. 

\vspace{0.25cm}
\tcp{Stage 1} 
\vspace{0.25cm}

\For{ $t \in [T']$}{
choose $\DO(\varnothing )$ \;
}
estimate $\mu_{\DO(\varnothing )}$, $\al$ as in Alg.~\ref{alg:nmc}

\vspace{0.25cm}
\tcp{Stage 2} 
\vspace{0.25cm}

\For{ $t \in [T'+1:2T']$}{
choose $\DO(S=1)$ \;
}

estimate $p_j$ as $\ph_j = \f 1 {T'} \sum_{t \in [T'+1:2 T']} X_j^{(t)}$\;

estimate $\mu_{s1j1}$, $\mu_{s1j0}$, $\mu_{\DO(S=1)}$ as $\muh_{s1j1}$, $\muh_{s1j0}$, $\muh_{\DO(S=1)}$\;

\For{ $t \in [2T'+1:3T']$}{
choose $\DO(S=0)$ \;
}

estimate $q_j$ as $\qh_j = \f 1 {T'} \sum_{t \in [2T'+1:3 T']} X_j^{(t)}$\;

estimate $\mu_{s0j1}$,  $\mu_{s0j0}$, $\mu_{\DO(S=0)}$ as  $\muh_{s0j1}$,  $\muh_{s0j0}$, $\muh_{\DO(S=0)}$\;

\vspace{0.25cm}
\tcp{Stage 3} 
\vspace{0.25cm}

let  $\Bcalh_{11}$, $\Bcalh_{10}$, $\Bcalh_{01}$, $\Bcalh_{00}$ be defined as in Alg.~\ref{alg:nmc}

$\uv_{11}$ = \FMain($\hat{\mathcal B}_{11}$, $1$, 1, $3T'$, $T'$) \;

$\uv_{10}$ = \FMain($ \hat{\mathcal B}_{10}$, 1, $0$, $6T'$, $T'$) \;

$\uv_{01}$ = \FMain($\hat{\mathcal B}_{01}$, 0, $1$, $9T'$, $T'$) \;

$\uv_{00}$ = \FMain($ \hat{\mathcal B}_{10}$, 0, $0$, $12T'$, $T'$) \;

\For{ $l\in\Scal,k\in\Xcal,$}{
update $\muh_{sl jk}$  with $\uv_{lk}$ for $j\in\hat{\mathcal B}_{lk}$ \;
}
\For{ $j \in [N]$}{
 $\muh_{\DO(X_j=1)}=\alh \muh_{s1j1}+(1-\alh)\muh_{s0j1}$ \;
 
 $\muh_{\DO(X_j=0)}=\alh \muh_{s1j0}+(1-\alh)\muh_{s0j0}$ \;
}

output 
$\hat{A}_T = \argmax_{a\in \mathcal A^{\textsf{mc}}}   \muh_{a}$

\caption{HCB-mc algorithm used in  Thm.~\ref{thm:upper bound}. The function {\tt{Refine}} is defined as in Alg.~\ref{alg:nmc}. }

\label{alg:mc}
\end{algorithm}

\begin{proof}

Since the agent is now able to manipulate the context node, we would like to utilize this additional action to further improve learning efficiency. 
To this end, we introduce Alg.~\ref{alg:mc}, a three-stage algorithm that estimates $\alpha$, $\pv$, $\qv$, $\mu_{\DO(\varnothing)}$, and $\mu_{\DO(S=l)}$, $\mu_{sljk}$, for all $j\in [N]$, $l\in\Scal$, $k\in\Xcal$. Here we provide an overview of the algorithm:
\begin{enumerate}
    \item In the first stage, the agent purely observes the system and estimates $\alpha$, $\mu_{\DO(\varnothing)}$. 
    \item In the second stage,  we conduct intervention $\DO(S=1)$ for the first half of the stage, obtaining estimates for $\mathbf{p}$, $m_1$, $\muh_{\DO(S=1)}$, and $\mu_{s1j1}$, $\mu_{s1j0}$ for $j\in[N]$. For the second half of this stage, conduct $\DO(S=0)$ and estimate the rest of the parameters.
    \item In the third stage, refine the estimates for the ``hard" entries by conducting interventions targeting them. This stage follows the same process as Stage 2 of Alg.~\ref{alg:nmc}.
\end{enumerate}
Here, let us   elaborate further on  the design of Stage 2 in Alg.~\ref{alg:mc}, as the analysis of Stages 1 and 3 are similar to Stages 1 and 2 of Alg.~\ref{alg:nmc}, respectively.

For the first half of Stage 2, we conduct action $\DO(S=1)$. 
Note that the value $\mu_{s1j1}$ is estimated when the event $\{X_j=1, S=1\}$ occurs. 
 When the intervention $S=1$ is performed, $\mu_{s1j1}$ is estimated whenever $\{X_j=1\}$ occurs. 
 Let $\alpha\leq 1/2$ for this discussion: with this assumption, the intervention $S=1$ 
 increases the chance of observing event $\{S=1,X_j=1\}$ compared to pure observation. 
 In particular, if $\alpha$ is highly biased towards zero, i.e., $\alpha\ll 1/2$, the intervention on $S=1$ significantly increases the efficiency of estimating $\mu_{s1j1}$. 

Similar to the non-manipulable-context case, we introduce a threshold determined by $m_1$. If the number of occurrences of event $\{X_j=1\}$ surpasses this threshold, we accept the estimate of $\mu_{s1j1}$ and categorize it as an ``easy'' entry; otherwise, we regard it as a ``hard'' entry and will refine its estimation in the next stage. Similar analysis goes with $\mu_{s1j0}$.

For the second half of this stage, we conduct action $\DO(S=0)$, and estimate $\mathbf{q}$, $m_0$, $\muh_{\DO(S=0)}$, and $\mu_{s0j1}$, $\mu_{s0j0}$ for $j\in[N]$. The analysis for the second half of Stage 2 is analogous to that of the first stage. 

Finally, for the third stage, we conduct interventions targeted to each of the ``hard'' entries and aims at refining the estimates obtained in the previous stage.
This is analogous to what we did in Stage 2 for the $\textsf{nmc}$ scenario (see Eqs.~\eqref{eq:stage2-1}, \eqref{eq:stage2-2}, and \eqref{eq:stage2-3}): accordingly the analysis following similar steps to those in the $\textsf{nmc}$ model may be omitted in the proof.  

\noindent
\underline{\bf Stage 1:}
In this stage, we estimate $\alpha$ and $\mu_{\DO(\varnothing)}$ using $T' = T/15$ experiments of pure observation. By Chernoff bound, we have the same concentration inequalities as Eqs.~\eqref{eq:alpha} and \eqref{eq:observe}, i.e., 
\begin{align}
     P\left(\left|\hat{\alpha}-\alpha\right|  \geq \varepsilon_{3\alpha,2,T'} \right)\leq \frac{1}{T'} \\
     P\left(\left|\hat{\mu}_{\DO(\varnothing)}-\mu_{\DO(\varnothing)}\right|  \geq \varepsilon_{3,2,T'}\right)\leq \frac{1}{T'},
\end{align}
where $\hat{\alpha}$ and $\hat{\mu}_{\DO(\varnothing)}$ are defined in the same way as in the $\textsf{nmc}$ case. 

\noindent 
\underline{\bf Stage 2:} In this stage we conduct a total of $2T'$ experiments where $T'=T/15$. For the first half of Stage 2, conduct action $\DO(S=1)$ for $T'$ experiments. 
We can derive an estimate for each $p_i$, with
$$\hat{p}_i = \frac{1}{T'}\sum_{t\in [T'+1:2T']}X_i^{(t)} \quad \mbox{(line 9 of Alg.~\ref{alg:mc}}),$$
for $i\in[N]$, and where $\hat{\mathbf{p}} = \left(\hat{p}_1,\ldots,\hat{p}_{N}\right)$. 
In addition, in this stage we define estimators
\begin{equation}
\hat{\mu}_{\DO(S=1)} = \frac{1}{T'}\sum_{t\in [T'+1:2T']}Y^{(t)},
\end{equation} 
\begin{equation}
\hat{\mu}_{s1j1} = \frac{\sum_{t\in [T'+1:2T']}Y^{(t)}X_j^{(t)}}{\sum_{t\in [T'+1:2T']}X_j^{(t)}} \quad \mbox{(line 10)},
\end{equation}
Using the Chernoff bound, union bound, and the definition of $\hat{m}_1$, we have that,
when $T'$ is sufficiently large such that 
\begin{equation}\label{eq:condition_T}
   T'>48m_1\log(2NT'), 
\end{equation}
we have 
\begin{equation}\label{eq:union pj}
    P\left(\left|\hat{p}_j-p_j\right|\leq \varepsilon_{{3p_j},2N,T'},\ \forall j\in[N]\right) \geq 1- \frac{1}{T'},
\end{equation}
\begin{equation}
     P\left(\left|\hat{\mu}_{\DO(S=1)}-\mu_{\DO(S=1)}\right|\geq \varepsilon_{3,2,T'}\right)\leq \frac{1}{T'},
\end{equation}
\begin{equation}
    P\left(\frac{2m_1}{3}\leq \hat{m}_1\leq 2m_1\right)\geq 1-\frac{2}{T'},
\end{equation}
and 
\begin{equation}
    P\left( \hat{\mu}_{s1j1}\in \left[ \mu_{s1j1} \pm
    \varepsilon_{{135m_1}/{2},N,T'}\right],\quad \forall j\in[N]\setminus\hat{\mathcal B}_{11}\right)\geq 1-\frac{4}{T'}.
        \label{eq:me1j1}
\end{equation}
In Eq.~\eqref{eq:me1j1}, we have derived an estimate for the ``easy'' entries in $\mu_{s1j1}$. We can derive an estimate for the ``easy'' entries of $\mu_{s1j0}$ at the same time in the same way. Note that here, the $\varepsilon$ value does not depend on $\alpha$, Compared to the $\textsf{nmc}$ case, we see that the ability of intervening $S$ makes the error in this stage significantly smaller when $\alpha$ is highly biased to zero. 

For the second half of the Stage 2, we conduct action $\DO(S=0)$ for $T'$ times. The corresponding estimators can be studied in a manner analogous to that in the proof of the non-manipulable case.

\noindent
\underline{\bf Stage 3:} 
In this stage we conduct a total of $4T''$ experiments where $T''= T/5$, refining the estimates for the ``hard'' entries in $\mu_{sljk}$. Denote by $\muh^{(r)}_{sljk}$ the refined estimators. 
First, we use $T_3$ experiments to estimate $ \mu_{s1j1} $ for $j\in \hat{\mathcal B}_{11}$. Since the analysis is the same as Stage 2 in the $\textsf{nmc}$ case, we omit the intermediate steps, and present the result as follows:
\begin{equation}
    P\left( \hat{\mu}^{(r)}_{s1j1}\in \left[ \mu_{s1j1} \pm \varepsilon_{{54m_1}/{\alpha},2N,T''}\right],\quad \forall j\in \hat{\mathcal B}_{11}\right)\geq 1-\frac{2}{T'}-\frac{4m_1}{T''},
\end{equation}
when $T''$ satisfies $T''>\frac{54m_1\log (2NT'')}{\alpha}$.
Therefore, we can derive the following bound for $\muh_a$: 
\begin{align}
     P\left(|\hat{\mu}_a -\mu_a|\leq 15\sqrt{3}\left(\sqrt{m_1\alpha}+\sqrt{m_0\alb}\right)\sqrt{\frac{\log (NT)}{T}},\ \forall a\in\Acal\right) \geq  1-\frac{330+40(m_1+m_0)}{T},
\end{align}
when $T> 270\max\left(\frac{m_1}{\alpha},\frac{m_0}{\alb}\right)\log (NT)$. Now suppose $T$ satisfies $T\geq 540\max\left(\frac{m_1}{\alpha},\frac{m_0}{\alb}\right)\log (NT)$. 
Following the same steps as those leading to Eq.~\eqref{eq:Rtnmc}, we can obtain the regret upper bound: 
\begin{align}
    R_T 
   % \leq & 30\sqrt{3}\left(\sqrt{m_1\alpha}+\sqrt{m_0\alb}\right)\sqrt{\frac{\log (NT)}{T}}+\frac{330+40(m_1+m_0)}{T}\nonumber\\
%    \leq & 58\left(\sqrt{m_1\alpha}+\sqrt{m_0\alb}\right)\sqrt{\frac{\log (NT)}{T}}\nonumber\\
    \leq & 
    116 \sqrt{\frac{(m_1\alpha+m_0\alb)\log(NT)}{T}}
\end{align}
as in Eq.~\eqref{eq:upper bound mc} in Thm.~\ref{thm:upper bound}.
\end{proof}

\section{Proof of Theorem \ref{thm:lower_bound}}
\label{app:lower_bound}

In this section, let us prove the lower bound in Theorem \ref{thm:lower_bound}. 
Recall that, as in Sec. \ref{sec:Bernoulli HCB}, we define
$
\mathbf{p} = (p_1,\ldots,p_{N}),
$
and
$
\mathbf{q} = (q_1,\ldots,q_{N}). 
$
Additionally, without loss of generality, we assume in the following  that 
\ea{p_1\leq  \ldots \leq p_{N}\leq \frac{1}{2}.
\label{eq:ordered}
}
Note that no such assumption is made for the $q_i$ sequence. 
Recall from in Lem.~\ref{lem:small}, we showed that $m_1$ characterizes the number of highly-biased entries in $\pv$.
In particular, we have $p_1\leq  \ldots \leq p_{\lceil m_1\rceil }\leq \frac{1}{m_1}$ when  the assumption in \eqref{eq:ordered} holds. 

\begin{proof}
In this section, we will first focus on proving that $$R_T\geq \frac{1}{127}\sqrt{\frac{m_1\alpha^2}{T}}, \quad \mbox{for $m_1\geq \tau_1(\alpha)$}.$$ 
The case of $m_0\geq \tau_0(\alpha)$ can be proved following similar steps as the case of $m_1\geq \tau_1(\alpha)$. The case where $m_1\geq \tau_1(\alpha)$ and $m_0\geq \tau_0(\alpha)$ simultaneously is the intersection of the two cases above, which implies that the lower bounds  $R_T\geq\frac{1}{127}\sqrt{m_1\alpha^2/T}$ and $R_T\geq \frac{1}{127}\sqrt{m_0\overline{\alpha}^2/T}$ are both satisfied. 
Finally, we will also briefly discuss the proof for the case that $m_1<\tau_1(\alpha)$ and $m_0<\tau_0(\alpha)$ at the same time.

The proof of this theorem consists of the following three main steps. 

\begin{enumerate}
    \item First, we construct a series of reward functions $r_i$, and a baseline reward function $r_0$ such that $r_i$ is greater than $r_0$ by $\varepsilon$ if the variables $\Xv$ fall into a particular set, which we refer to as the target set; for everywhere else, $r_i$ equals $r_0$. Under this construction, the optimal action will be the one that maximizes the chance of hitting this target set.
    \item Next, we show that there exists an optimal action under which we can reach the target set with probability at least $\alpha/e$. If the agent chooses any other action, however, the chance of reaching this target set will be significantly lower. After proving the claim above, we know that the agent will suffer non-trivial regret if the optimal action is not chosen. 
    \item Finally, in the third part of the proof,   
we show that the probability that the agent chooses the wrong action is bounded from below by a constant. Combining the second and third parts of the proof, we know that the agent will choose a wrong action with a non-negligible probability, and if this happens, the regret is non-trivial, thus completing the proof of the lower bound. 
\end{enumerate}

Note that for the HCB model, each $X_i$ is a parent node of $Y$, i.e., $ \mathrm{pa}(Y)  = [N]$. We begin by constructing the following series of reward functions.  Let $r_0$, $r_i$, $i\in[\lceil m_1\rceil]$ be reward functions with 
\begin{equation}
r_0\left(\mathbf{x}\right)=\frac{1}{2},
\end{equation}
for all $\mathbf{x}\in \Xcal^N$, and 
\begin{equation}
r_i(\mathbf{x})=\frac{1}{2}+\varepsilon\mathbf{1}_{\{\mathbf{x}\in\Xcal_i^*\}},
\end{equation}
where $$\Xcal_i^* = \left\{\mathbf{x}\in\Xcal^N: x_i = 1, x_\ell = 0, \ell\in [\lceil m_1 \rceil]\backslash\{i\}\right\},$$ is the target set defined in Eq.~\eqref{def:X_star}, for $i\in[\lceil m_1\rceil]$. Here, $\varepsilon\in(0,1/4)$ is a parameter to be specified later in the proof. In other words, the function $r_i$ favors a specific state of $\mathbf{x}$, one in which the $i$th entry is one while all other in the first $\lceil m_1\rceil$ entries are zero. 

Let $P_j$ be the probability distribution of $\Hv_T$ when the system parameters are  $(\alpha,\mathbf{p},\mathbf{q},r_j)$ and the algorithm is $\phi$. Let $\Ebb_j$ be the expectation with respect to $P_j$, for $j\in [0:\lceil m_1\rceil ]$.  
In the following lemma, we show that for $i\in [\lceil m_1\rceil]$, when the reward function is fixed as $r_i$, the conditional reward given action $a$ is determined by the probability of $\Xv$ reaching the set $\Xcal_i^*$. The proof of this lemma is provided in App.~\ref{app:additional3}.
\begin{lemma}\label{lem:additional3}
Fix $i\in [\lceil m_1\rceil]$, and suppose that the reward function is $r_i$. 
The conditional reward given an action $a$ can be expanded as
\begin{align}
    P_i(Y=1|A=a) 
    = & \frac{1}{2}+\varepsilon P_0(\Xv\in\Xcal_i^*|A=a).
\end{align}
\end{lemma}
Now we show that under reward function $r_i$, there exists an optimal action $a_i^*$  such that by taking this action, the variables $\Xv$ are much more likely to reach the target set than they would under any other action.  
The significance of this fact, if true, is that then if the agent fails to select this optimal action, the resulting loss in expected reward would be bounded from below by $\varepsilon$ times a constant value. 

With Lem.~\ref{lem:additional3}, we see that the expected reward of an action $a$ is determined by the conditional probability that $\Xv$ falls into the set $\Xcal_i^*$ given the intervention $A=a$.  Now we further study the term $P_0(\Xv\in\Xcal_i^*|A=a)$. 
In Lem.~\ref{lem:1}, we showed that:
\begin{enumerate}
\item the probability of reaching the target set $\Xcal_i^*$ under action $\DO(X_i=1)$ is at least $\alpha/e$, i.e., $$P_0(\Xv\in\Xcal_i^*|\DO(X_i=1))\geq \frac{\alpha}{e},$$ 

\item the probability of reaching that set under any other action is at most $1/m_1$, i.e.,  
$$P_0(\Xv\in\Xcal_i^*|A=a)\leq \frac{1}{m_1}, \quad a\neq \DO(X_i=1).$$ 
\end{enumerate}
Since $m_1\geq \frac{3e}{(3-e)\alpha}$, we have $\frac{\alpha}{e}> \frac{1}{m_1}$, which implies that the optimal action corresponding to reward function $r_i$ is 
\begin{equation}
    a_i^* = \DO(X_i=1). 
\end{equation}
Now we are able to show an intermediate lower bound for the regret in the following lemma, which is proved in App.~\ref{app:intermediate}.
\begin{lemma}\label{lem:intermediate}
Fix $i\in[\lceil m_1\rceil]$ and let the reward function be $r_i$. The simple regret $R_T$ satisfies 
\begin{equation}
    R_T\geq \varepsilon  \left(\frac{\alpha}{e}-\frac{1}{m_1}\right)P_{i} (\hat{A}_T \neq a_{i}^*) 
\end{equation}
for all $T\geq 1$.
\end{lemma}
With the lemma above, we see that assuming the wrong action is chosen, i.e., $\hat{A}_T \neq a_i^*$, the regret incurred by this choice is at least $\varepsilon\left(\frac{\alpha}{e}-\frac{1}{m_1}\right)$.

Now it remains to characterize the probability with which the agent fails to identify the optimal action, i.e., $P_i(\hat{A}_T\neq a_i^*)$ when the reward function is $r_i$. 
If this probability is bounded from below by a constant, then we will have obtained a non-trivial lower bound based on Lem.~\ref{lem:intermediate}. This portion of the proof makes use of relatively standard techniques in the literature for deriving such lower bounds, that is, by showing that the distributions $P_0$ and $P_i$ are highly similar (measured by KL-divergence) and thus difficult to be distinguished via a finite number of samples: see for instance \cite{tsybakov2008introduction} and \cite{gerchinovitz2016refined}. 

In this part of the proof, we will first prove an upper bound on the KL divergence between $P_0$ and $P_i$. After that, with Lem.~\ref{lemma:subset}, we are able to obtain a lower bound on $P_0(B)+P_i(B^c)$, where $B$ is any measurable event. Let $B=\{\hat{A}_T=a_i^*\}$ be the event that the agent chooses action $a_i^*$, then its complement $B^c=\{\hat{A}_T\neq a_i^*\}$ is the event that $a_i^*$ is not selected. As a result, we will be able to derive a lower bound for its probability under distribution $P_i$. Thus, the missing piece of the lower bound in Lem.~\ref{lem:intermediate} will be completed. 

Let us next present the more formal derivation of the bound. 
Fix an arbitrary $i\in[\lceil m_1\rceil]$ and let us consider the KL divergence between distributions $P_0$ and $P_i$, which we denote as
\begin{align}
    D_{\rm KL}(P_0,P_i) 
    = & D_{\rm KL}(P_0(\Hv_T),P_i(\Hv_T)).
    \label{eq:KL}
\end{align}
In Lem.~\ref{lemma:kl}, we showed that 
$D_{\rm KL}(P_0,P_i)$ can be decomposed along the time horizon, i.e., 
\begin{align}
    &D_{\rm KL}(P_0(\Hv_T),P_i(\Hv_T)) \nonumber\\
    = &\sum_{t\in[T]} \sum_{\mathbf{x}^{(t)}\in\Xcal^{N}} P_0\left(\mathbf{x}^{(t)}\right) \left(r_0\left(\mathbf{x}^{(t)}\right)\log\frac{r_0\left(\mathbf{x}^{(t)}\right)}{r_i\left(\mathbf{x}^{(t)}\right)}+\left(1-r_0\left(\mathbf{x}^{(t)}\right)\right)\log\frac{1-r_0\left(\mathbf{x}^{(t)}\right)}{1-r_i\left(\mathbf{x}^{(t)}\right)}\right).
\end{align}
In the following lemma, we derive an upper bound for the KL divergence based on the decomposition above. This upper bound is determined by the probability of reaching the target set $\Xcal_i^*$. We postpone the proof of this lemma to App.~\ref{app:additional4}.
\begin{lemma}\label{lem:additional4}
The KL divergence between $P_0$ and $P_i$ in \eqref{eq:KL} satisfies 
\begin{equation}
    D_{\rm KL}(P_0(\Hv_T),P_i(\Hv_T)) \leq \sum_{t\in [T]}P_0(\Xv^{(t)}\in\mathcal{X}_i^*)\frac{16\varepsilon^2}{3}.
\end{equation}
\end{lemma}

With Lem.~\ref{lem:additional4} in mind, we can expand the KL divergence as follows:
    \begin{align}
     D_{\rm KL} (P_0(\Hv_T),P_i(\Hv_T))\leq & \frac{16\varepsilon^2}{3}\sum_{t\in[T]} P_0(A^{(t)} = a_i^*)+\frac{16\varepsilon^2}{3m_1}\sum_{t\in[T]} P_0(A^{(t)} \neq a_i^*)\label{pf:firstline}\\
    = & \frac{16\varepsilon^2}{3}\left(1-\frac{1}{m_1}\right)\sum_{t\in[T]} P_0(A^{(t)} = a_i^*)+\frac{16\varepsilon^2T}{3m_1}\nonumber\\
    = & \frac{16\varepsilon^2}{3}\left(1-\frac{1}{m_1}\right)\Ebb_0\left[\sum_{t\in[T]}\mathbf{1}_{\{A^{(t)} =
    a_i^*\}}\right]+\frac{16\varepsilon^2T}{3m_1},\label{eq:klm}
\end{align}

where in Eq.~\eqref{pf:firstline}, we applied Lem.~\ref{lem:1}: in particular, we used the property of the system that for any action other than $a_i^*$, the probability of reaching the target set is at most $1/m_1$. 
Also, in \eqref{eq:klm}, $\Ebb_i$ with $i \in [0:N]$ indicates the expected value according to the distribution $P_i$.

In the next lemma, we show that there exists a reward function $r_i$, under which the agent will choose the wrong action at the end of the experiment with non-negligible probability. To prove this lemma, we will take advantage of the upper bound on the KL divergence in Eq.~\eqref{eq:klm}. In particular, we will use Eq.~\eqref{eq:klm} to derive a lower bound for the event that the agent mistakes a sub-optimal action for the optimal one. The proof of the lemma is delayed until App.~\ref{app:miss}.
\begin{lemma}\label{lem:miss}
Fix parameters $\alpha$, $\pv$, $\qv$ and fix an algorithm $\phi$ defined in Eq.~\eqref{eq:algo}. 
There exists $i'\in [\lceil m_1\rceil]$, such that for reward function $r_{i'}$, 
\begin{align}
    P_{i'}\left(\hat{A}_{T} \neq a_{i'}^*\right)\geq \frac{1}{2.1}-\frac{2}{m_1}.
\end{align}
\end{lemma}
Combining Lem.~\ref{lem:miss} with Lem.~\ref{lem:intermediate}, we are now in a position to complete the proof.
Since $m_1\geq \frac{3e}{3-e}\frac{1}{\alpha}$, we have that there must be an index $i'$ such that if we let the reward function be $r_{i'}$,
\begin{align}
R_T 
\geq &\varepsilon \left(\frac{1}{2.1}-\frac{2}{m_1}\right)\left(\frac{\alpha}{e}-\frac{1}{m_1}\right)\geq \frac{19\alpha\varepsilon}{126} = \frac{1}{127}\sqrt{\frac{m_1\alpha^2}{T}}. 
\end{align}
If $m_0\geq\frac{3e}{3-e}\frac{1}{\alb}$, we similarly  have  that
\begin{equation}
    R_T \geq \frac{1}{127}\sqrt{\frac{m_0\alb^2}{T}}.
\end{equation}

In summary, we have now shown the lower bounds for the case of $m_1\geq \tau_1(\alpha)$ and that of $m_0\geq \tau_0(\alpha)$. It follows that if $m_1\geq \tau_1(\alpha)$ and $m_0\geq \tau_0(\alpha)$ hold at the same time, the regret must satisfy both of the those lower bounds, i.e., 
\begin{equation}
    R_T \geq \frac{1}{127}\sqrt{\frac{\max(m_1\alpha^2,m_0\alb^2)}{T}}.
\end{equation}  

Now it remains to prove for the case that $m_1<\tau_1(\alpha)$ and $m_0<\tau_0(\alpha)$ simultaneously. 
To conclude, we briefly explain this scenario: 

Choose 
$$\Xcal_i^* = \{\mathbf{x}\in \Xcal^{ N }:x_i=1\}.$$ 
With Lem.~\ref{lem:additional4}, the KL divergence is bounded by 
\begin{equation}
     D_{\rm KL} (P_0(\Hv_T),P_i(\Hv_T))\
     \leq  \frac{16\varepsilon^2}{3}\sum_{t\in[T]} P_0(\Xv^{(t)}\in\mathcal{X}_i^*)\leq \frac{16\varepsilon^2T}{3}.
     \label{eq:less less}
\end{equation}
Choose $\varepsilon$ in Eq.~\eqref{eq:less less} as
\begin{equation}
\varepsilon = \frac{\sqrt{\ln (1.05)}}{4}\sqrt{\frac{1}{T}}.
\end{equation}
Following similar steps as in the proof of Lem.~\ref{lem:miss}, we have that there exists $i'\in [N]$ such that 
\begin{align}
    P_{i'}\left(\hat{A}_T \neq a_{i'}^*\right)\geq \frac{1}{7}.
\end{align}
In other words, when the reward function is $r_{i'}$, the probability of selecting a wrong action is at least $1/7$. 
In addition, the regret of choosing a wrong action is at least 
\begin{equation}
    P_{i'}(\Xv\in\Xcal_{i'}^*|A=a_{i'}^*)-\max_{a\neq a_{i'}^*} P_{i'}(\Xv\in\Xcal_{i'}^*|A=a))\geq 1-\max\left(q_{\max},\frac{1}{2}\right).
\end{equation}
Therefore, following similar steps as in the previous case, we have the following lower bound for the simple regret
\begin{equation}
    R_T \geq \frac{1-\max\left(q_{\max},\frac{1}{2}\right)}{127}\sqrt{\frac{1}{T}}.
\end{equation}
This concludes the proof.
\end{proof}

\section{Technical Analysis}\label{app:tech}

\subsection{Proof of Lemma \ref{lem:exists}}
\label{app:exists}

The proof involves showing that the minimum in Eq.~\eqref{def:m} exists.
Let $\vv = (v_1,\ldots,v_N)$, and let $\bm \theta = (\theta_1,\ldots,\theta_N)$, where $\theta_i =\min(v_i,\overline{v_i})$. Without loss of generality, suppose $\theta_1\leq \ldots\leq \theta_N\leq \frac{1}{2}$. Although we restrict that $\vv\in(0,1)^N$ throughout the paper, in this lemma we may prove for the more general case that $\vv\in[0,1]^N$, which corresponds to $\bm \theta = [0,1/2]^N$.

Now we suppose that $\vv$ is a given vector, and let $f(s) = |\mathcal I_s(\vv)|$ be a function of $s$ for $s\in[1,\infty)$.
We show that $f(s)$ is right-continuous, i.e., $f(u) = \lim_{s\to u+} f(s)$ for $u\in[1,\infty)$. 
Let $k_0$ be a non-negative integer. Suppose $\theta_i = 0$ for $i\leq k_0$, that is $k_0$ indicates the number of entries in $\bm \theta$ which are equal to zero.
Note that when $k_0=0$, we have $\theta_i\in(0,1/2)$ for $i\in[N]$.
 It is easy to see that $0\leq f(s)\leq N$.
In particular, the function $f(s)$ can be expanded as follows:
\begin{equation}
f(s) = |\{i\in[N]:\theta_i<1/s\}| = k_0+|\{i\in[k_0+1:N]:\theta_i<1/s\}|.
\end{equation}

Suppose the entries $\theta_i$, $i\in[k_0+1:N]$ take $n$ different values $1/u_1>\ldots>1/u_n$, with multiplicity $k_1,\ldots,k_n$, respectively. In other words, 
\begin{equation}
k_j=\sum_{i\in[k_0+1:N]}\mathbf{1}_{\{ v_i = 1/u_j\}},
\end{equation}
where $u_i\geq 2$, $i\in[n]$, and  $k_0+\sum_{j\in[n]} k_j = N$. 
It is easy to see that 
\begin{equation}
    f(s) = 
    \begin{cases}
    N & s\in [0,u_1)\\
    N-\sum_{j\in[i]} k_j & s\in [u_i,u_{i+1}), i\in[n-1]\\
    k_0 & s\in [u_n,\infty).
    \end{cases}
\end{equation}
Therefore, $f:[1,\infty)\mapsto [0:N]$ is a non-increasing step function, whose right limit exists everywhere on its domain. 

We are now ready to prove that the minimum in Eq.~\eqref{def:m} exists. Let $\mathcal  J = \{s\in \mathbb R_+: s\geq f(s)\}$. It is not difficult to show that $\mathcal J$ is not empty. In particular, $s\in \mathcal J$ for all $s\geq N$, so $\mathcal J\neq \varnothing$. 

Next, it is easy to verify that $\mathcal J$ is a (connected) interval of the real line.  
In particular, suppose that $u\in \mathcal J$, then $u\geq f(u)$. For any $s\geq u$, we have  $s\geq u\geq f(u)\geq f(s)$, which implies that $s\in \mathcal J$. 

Finally, we show that $\mathcal J$ is a left-closed interval. Suppose $\mathcal J$ is left-open. Let $\mathcal J=(u,\infty)$, which implies that $s\geq f(s)$ for any $s>u$, while $u<f(u)$.  
Let $f(u) = u+\delta$ where $\delta>0$. 
Since $f$ is right-continuous, we have $u+\delta = f(u) = \lim_{s\to u+}f(s)$. 
Consequently, $u+\delta = \lim_{s\to u+}f(s)\leq \lim_{s\to u+}s = u$. This leads to $\delta\leq 0$, which is a contradiction. Therefore, we have shown that $\mathcal J$ is a left-closed interval, in which the minimum always exists.

\subsection{Proof of Lemma \ref{lem:small}}
\label{app:small}
For simplicity of notation, we use the shorthand $m\triangleq m(\mathbf{v})$ in this section.
For any $\delta > 0$, by the definition of $m$ in Eq.~\eqref{def:m}, we have $|\mathcal I_{m-\delta}(\vv)|>m-\delta$. 
We will use this observation to  prove this lemma by contradiction. 

Suppose that $v_{\lceil m\rceil}>\frac{1}{m}$. In particular, let $v_{\lceil m\rceil} = \frac{1}{m-\delta_1}$ for some $\delta_1>0$. 
We consider the following two cases separately: (1) $m-\delta_1\geq \lceil m\rceil -1$, and (2)  $m-\delta_1< \lceil m\rceil -1$. 

\noindent (1) First suppose $m-\delta_1\geq \lceil m\rceil -1$.
Note that $|\mathcal I_{m-\delta_1}(\vv)|> m-\delta_1$ but also $ |\mathcal I_{m-\delta_1}(\vv)|\leq \lceil m\rceil -1$. This leads to a contradiction. 

\noindent (2) Next suppose $m-\delta_1< \lceil m\rceil -1$, then $v_{\lceil m\rceil}>\frac{1}{\lceil m\rceil -1}$, and therefore
$\left|\mathcal I_{ \lceil m\rceil -1}(\vv)\right|\leq \lceil m\rceil -1$.
Note also that $\lceil m\rceil-1<m$. Consequently, we have
$
     \lceil m\rceil -1< \left|\mathcal I_{ \lceil m\rceil -1}(\vv)\right|,
$
which is again a contradiction. 

In summary, we have shown that $v_{\lceil m\rceil}\leq\frac{1}{m}$.

\subsection{Proof of Lemma \ref{lemma:intervention}}\label{app:intervention}
\begin{proof}
For $i\in [N]$, let $\mathcal N_i = [N]\backslash\{i\}$. 
We can expand $P(Y=1|\DO(X_i=1))$ as follows:
\begin{align}
& P(Y=1|\DO(X_i=1))\nonumber\\ 
= &\sum_{\mathbf{x}_{\mathcal N_i}\in\Xcal^{N-1}}\ \ \ \sum_{s\in\Scal}P\left(Y=1|X_i=1,\mathbf{X}_{\mathcal N_i}=\mathbf{x}_{\mathcal N_i}\right)P\left(\Xv_{\mathcal N_i}=\mathbf{x}_{\mathcal N_i}|S=s\right)P(S=s)\nonumber\\
= &\sum_{s\in\Scal} \ \ \ \sum_{\mathbf{x}_{\mathcal N_i}\in\Xcal^{N-1}}P\left(Y=1|X_i=1,\Xv_{\mathcal N_i}=\mathbf{x}_{\mathcal N_i},S=s\right)P\left(\mathbf{X}_{\mathcal N_i}=\mathbf{x}_{\mathcal N_i}|S=s\right)P(S=s)\label{pf:Sind}\\
= &\sum_{s\in\Scal}\ \ \ \sum_{\mathbf{x}_{\mathcal N_i}\in\Xcal^{N-1}}\frac{P\left(Y=1,X_i=1,\Xv_{\mathcal N_i}=\mathbf{x}_{\mathcal N_i},S=s\right)}{P\left(X_i=1|\Xv_{\mathcal N_i}=\mathbf{x}_{\mathcal N_i},S=s\right)}\nonumber\\
= &\sum_{s\in\Scal}\ \ \ \sum_{\mathbf{x}_{\mathcal N_i}\in\Xcal^{N-1}}\frac{P\left(Y=1,X_i=1,\Xv_{\mathcal N_i}=\mathbf{x}_{\mathcal N_i},S=s\right)}{P\left(X_i=1|S=s\right)}\label{pf:Xind}\\
= &\sum_{s\in\Scal}\frac{P\left(Y=1,X_i=1,S=s\right)}{P\left(X_i=1|S=s\right)}\nonumber\\
= & P(Y=1|X_i=1,S=1)P(S=1)+P(Y=1|X_i=1,S=0)P(S=0),\label{pf:expand}
\end{align} 
for $i\in [N]$. 
Here, Eqs.~\eqref{pf:Sind} and \eqref{pf:expand} follow from the fact that $Y$ is conditionally independent of $S$ given $\Xv$, and Eq.~\eqref{pf:Xind} from that the entries in $\Xv$ are mutually independent conditioning on $S$. 
Similarly, we can prove the case for $P(Y=1|\DO(X_i=0))$. Eqs.~\eqref{eq:do2} and \eqref{eq:do3} can also be shown similarly.  
\end{proof}

\subsection{Proof of Lemma \ref{lemma:kl}}\label{app:kl}

\begin{proof}
 By the chain rule of KL divergence in Lem.~\ref{lemma:kl}, we first decompose the KL divergence between $P_0$ and $P_i$ as follows,
\begin{align}
     & D_{\rm KL} (P_0(\Hv_T),P_i(\Hv_T))\nonumber\\
    = &  D_{\rm KL}(P_0(\Hv_{T-1}),P_i(\Hv_{T-1})) \nonumber\\
    & + D_{\rm KL}\left(P_0\left(S^{(T)},\Xv^{(T)},Y^{(T)},A^{(T)}|\Hv_{T-1}\right),P_i\left(S^{(T)},\Xv^{(T)},Y^{(T)},A^{(T)}|\Hv_{T-1}\right)\right).
\end{align}
For $j\in\{0,i\}$ and $t\in [T]$, the conditional distribution of the variables given the history can be expanded as follows,
\begin{align}\label{eq:klexpand}
& P_j(S^{(t)},\Xv^{(t)},Y^{(t)},A^{(t)}|\Hv_{t-1}) \\
= & \frac{P_j(Y^{(t)}|\Xv^{(t)},A^{(t)},S^{(t)},\Hv_{t-1})P_j(\Xv^{(t)},S^{(t)}|A^{(t)},\Hv_{t-1})P_j(A^{(t)},\Hv_{t-1})}{P_j(\Hv_{t-1})}\nonumber\\
= &  P_j(Y^{(t)}|\Xv^{(t)})P_j(\Xv^{(t)},S^{(t)}|A^{(t)})P_j(A^{(t)}|\Hv_{t-1}),
\label{eq:non cond}
\end{align}
where, in \eqref{eq:non cond}, we utilized the following two facts: 
\begin{enumerate}
\item
the reward at time $t$, $Y^{(t)}$, is conditionally independent of the action $A^{(t)}$, the context $S^{(t)}$, and the previous history $\Hv_{t-1}$, as long as $\Xv^{(t)}$ is given; 
\item 
the variables $(\Xv^{(t)},S^{(t)})$ are independent of the history $\Hv_{t-1}$ conditioning on the choice of action $A^{(t)}$.
\end{enumerate}
For any time step $t$, we can expand the KL divergence as follows,
\begin{align}
    &D_{\rm KL}\left(P_0\left(\Xv^{(t)},Y^{(t)},A^{(t)}, S^{(t)}|\Hv_{t-1}\right),P_i\left(\Xv^{(t)},Y^{(t)},A^{(t)},S^{(t)}|\Hv_{t-1}\right)\right) \nonumber\\
    = & \sum_{\hv_{t-1}}\sum_{\mathbf{x}^{(t)},y^{(t)},a^{(t)},s^{(t)}}P_0\left(\mathbf{x}^{(t)},y^{(t)},a^{(t)},s^{(t)},\hv_{t-1}\right)\log \frac{P_0\left(\mathbf{x}^{(t)}, y^{(t)},a^{(t)},s^{(t)}|\hv_{t-1}\right)}{P_i\left(\mathbf{x}^{(t)}, y^{(t)},a^{(t)},s^{(t)}|\hv_{t-1}\right)}\nonumber\\
    = & \sum_{\hv_{t-1}}\sum_{\mathbf{x}^{(t)},y^{(t)},a^{(t)},s^{(t)}}P_0\left(\mathbf{x}^{(t)},y^{(t)},a^{(t)},s^{(t)},\hv_{t-1}\right)\log \frac{ P_0\left(y^{(t)}|\mathbf{x}^{(t)}\right)P_0\left(\mathbf{x}^{(t)},s^{(t)}|a^{(t)}\right)P_0\left(a^{(t)}|\hv_{t-1}\right)}{ P_i\left(y^{(t)}|\mathbf{x}^{(t)}\right)P_i\left(\mathbf{x}^{(t)},s^{(t)}|a^{(t)}\right)P_i\left(a^{(t)}|\hv_{t-1}\right)},
    \end{align}
    where in the last line we applied Eq.~\eqref{eq:klexpand}. 
    Note that changing the reward function does not affect the conditional distribution of $\Xv$ and $S^{(t)}$ given the action. We have $P_i\left(\mathbf{x}^{(t)},s^{(t)}|a^{(t)}\right) = P_0\left(\mathbf{x}^{(t)},s^{(t)}|a^{(t)}\right)$. Note also that the algorithm $\phi$ is fixed. We have that the conditional distribution of the action chosen at time $t$ given the history does not depend on the reward function, which implies that  
    $P_i\left(a^{(t)}|\hv_{t-1}\right)  = P_0\left(a^{(t)}|\hv_{t-1}\right)$.
    We can further expand the KL divergence as follows, 
    \begin{align}
        &D_{\rm KL}\left(P_0\left(\Xv^{(t)},Y^{(t)},A^{(t)}, S^{(t)}|\Hv_{t-1}\right),P_i\left(\Xv^{(t)},Y^{(t)},A^{(t)},S^{(t)}|\Hv_{t-1}\right)\right) \nonumber\\
    = &\sum_{\hv_{t-1}}\sum_{\mathbf{x}^{(t)},y^{(t)},a^{(t)},s^{(t)}}P_0\left(\mathbf{x}^{(t)},y^{(t)},a^{(t)},s^{(t)},\hv_{t-1}\right)\log \frac{ P_0\left(y^{(t)}|\mathbf{x}^{(t)}\right)}{ P_i\left(y^{(t)}|\mathbf{x}^{(t)}\right)}\nonumber\\
    = &\sum_{\mathbf{x}^{(t)},y^{(t)}}P_0\left(\mathbf{x}^{(t)},y^{(t)}\right)\log \frac{ P_0\left(y^{(t)}|\mathbf{x}^{(t)}\right)}{ P_i\left(y^{(t)}|\mathbf{x}^{(t)}\right)}\nonumber\\
    = &\sum_{\mathbf{x}^{(t)}}P_0\left(\mathbf{x}^{(t)} \right)\sum_{y^{(t)}}P_0\left(y^{(t)}|\mathbf{x}^{(t)}\right)\log \frac{ P_0\left(y^{(t)}|\mathbf{x}^{(t)}\right)}{ P_i\left(y^{(t)}|\mathbf{x}^{(t)}\right)}.
\end{align}
Consequently, the decomposition of the KL divergence can be expanded as 
\begin{align}
     & D_{\rm KL} (P_0(\Hv_T),P_i(\Hv_T))\nonumber\\
    = &  D_{\rm KL}(P_0(\Hv_{T-1}),P_i(\Hv_{T-1})) + D_{\rm KL}\left(P_0\left(\Xv^{(T)},Y^{(T)},A^{(T)},S^{(T)}|\Hv_{T-1}\right),P_i\left(\Xv^{(T)},Y^{(T)},A^{(T)},S^{(T)}|\Hv_{T-1}\right)\right)\nonumber\\
     = &  D_{\rm KL}(P_0(\Hv_{T-1}),P_i(\Hv_{T-1})) \nonumber\\
     &+ \sum_{\mathbf{x}^{(T)}\in\Xcal^N} P_0\left(\mathbf{x}^{(T)}\right) \left(r_0\left(\mathbf{x}^{(T)}\right)\log\frac{r_0\left(\mathbf{x}^{(T)}\right)}{r_i\left(\mathbf{x}^{(T)}\right)}+\left(1-r_0\left(\mathbf{x}^{(T)}\right)\right)\log\frac{1-r_0\left(\mathbf{x}^{(T)}\right)}{1-r_i\left(\mathbf{x}^{(T)}\right)}\right)\nonumber\\
      = &  \sum_{t\in[T]} \sum_{\mathbf{x}^{(t)}\in\Xcal^{N}} P_0\left(\mathbf{x}^{(t)}\right) \left(r_0\left(\mathbf{x}^{(t)}\right)\log\frac{r_0\left(\mathbf{x}^{(t)}\right)}{r_i\left(\mathbf{x}^{(t)}\right)}+\left(1-r_0\left(\mathbf{x}^{(t)}\right)\right)\log\frac{1-r_0\left(\mathbf{x}^{(t)}\right)}{1-r_i\left(\mathbf{x}^{(t)}\right)}\right), 
\end{align}
which completes the proof. 
\end{proof}

\subsection{Proof of Lemma \ref{lem:1}}
\label{app:1}
In this section, we detail the proof of Lemma \ref{lem:1}.

\begin{proof}

We begin by showing the first part of the lemma, i.e., the analysis involving action $\DO(X_i=1)$ for $i\in [\lceil m_1\rceil]$.  Fix $i\in [\lceil m_1\rceil]$.
For action $a = \DO(X_i=1)$, we have 
\begin{align}
& P(\Xv\in \Xcal_i^*|\DO(X_i=1))\nonumber\\
 = & P \left(X_i = 1,X_{\ell}=0,\ell\in[\lceil m_1\rceil]\backslash \{i\}|\DO(X_i=1)\right) \nonumber\\
= & P(S=1)P(X_\ell=0,\ell\in[\lceil m_1\rceil]\backslash\{i\}|S=1) + P(S=0)P(X_\ell=0,\ell\in[\lceil m_1\rceil]\backslash\{i\}|S=0)\nonumber\\
= & \alpha \prod_{\ell\in[\lceil m_1\rceil]\backslash\{i\}}\po_\ell + \alb \prod_{\ell\in[\lceil m_1\rceil]\backslash\{i\}}\qo_\ell\nonumber\\
\geq & \alpha \prod_{\ell\in[\lceil m_1\rceil]\backslash\{i\}}\po_\ell.
\end{align}
Here, the inequality in the last line is simply obtained by dropping the second term.  
In Lem.~\ref{lem:small}, we showed that $p_1\leq p_2\leq\ldots \leq p_{\lceil m_1\rceil }\leq \frac{1}{m_1}$. As a result,
\begin{align}
    P(\Xv\in \Xcal_i^*|\DO(X_i=1)) \geq & \alpha \left(1-\frac{1}{m_1}\right)^{\lceil m_1\rceil-1}\geq \frac{\alpha}{e},
\end{align}
for each $i\in [\lceil m_1 \rceil]$, which proves the first part of the lemma. 

Now we prove the second part of the lemma, i.e, showing that for each $i\in \mathcal I$, the probability of reaching the target set $\Xcal_i^*$ is at most $1/m_1$ for any other action. 

Recall that $
\mathcal I = \left\{\pi(j), 1\leq j\leq \left\lfloor\frac{m(\pv) }{2}\right\rfloor\right\}
$, which is defined in the statement of the lemma. This is the set of indices corresponding to the smaller entries of the $\qv$ vector. 
Here, $\pi$ is the permutation that sorts the entries of $\qv$ in increasing order. 

Fix any $i\in \mathcal I$. In what follows in this proof, we will check for each action in $\mathcal A^{\textsf{mc}}$ other than $\DO(X_i=1)$. Note that the analysis also holds for the $\textsf{nmc}$ model since $\mathcal A^{\textsf{nmc}}\subseteq \mathcal A^{\textsf{mc}}$. In other words, if this lemma is true for the $\textsf{mc}$ model, it must also be true for the $\textsf{nmc}$ model.

For action $\DO(X_i=0)$, it is easy to see that 
\begin{equation}
P(\Xv\in \Xcal_i^*|\DO(X_i=0))=0. 
\end{equation}
For action $\DO(S=1)$, we have 
\begin{align}
    P\left(\Xv\in \Xcal_i^*|\DO(S=1)\right) = p_i\prod_{\ell\in[\lceil m_1\rceil]\backslash\{i\}}\po_\ell\leq \frac{1}{m_1}.
\end{align}
For action $\DO(S=0)$, 
\begin{align}
    P\left(\Xv\in \Xcal_i^*|\DO(S=0)\right) 
    = & q_i\prod_{\ell\in[\lceil m_1\rceil]\backslash\{i\}}\qo_\ell.
\end{align}
Recall that $i\in \mathcal I$, we have $q_i\leq q_{\pi(\ell)}$ for any $\ell\geq \left\lfloor\frac{m_1}{2}\right\rfloor+1$. Hence, 
\begin{align}
    P\left(\Xv\in \Xcal_{i}^*|\DO(S=0)\right) 
    = & q_{i}\prod_{\ell\in[\lceil m_1\rceil]\backslash\{i\}}\overline{q}_\ell\nonumber\\
    \leq & q_{i} \prod_{\ell = \left\lfloor\frac{m_1}{2}\right\rfloor+1}^{\lceil m_1\rceil}\overline{q}_{\pi(\ell)}\nonumber\\
    \leq & q_{i}\overline{q}_{i}^{\lceil m_1\rceil-\left\lfloor\frac{m_1}{2}\right\rfloor}\label{pf:pi}\\
    \leq & \frac{1}{\lceil m_1\rceil-\left\lfloor \frac{m_1}{2}\right\rfloor+1}\left(1-\frac{1}{\lceil m_1\rceil-\left\lfloor \frac{m_1}{2}\right\rfloor+1}\right)^{\lceil m_1\rceil-\left\lfloor\frac{m_1}{2}\right\rfloor}\label{pf:second-last}\\
    \leq & \frac{1}{m_1}.
\end{align}
In Eq.~\eqref{pf:pi}, we used the fact that $q_{\pi(1)}\leq q_{\pi(2)}\leq \ldots q_{\pi(\lceil m_1\rceil )}$ by the definition of the permutation $\pi$ first introduced in Lem.~\ref{lem:1}.
In Eq.~\eqref{pf:second-last}, we used the fact that the function $x(1-x)^n$ is maximized at $x=1/(n+1)$ for integer $n\geq 1$. 

Similarly, for $j\in [\lceil m_1\rceil]\backslash\{i\}$, the action $\DO(X_j=0)$ corresponds to 
\begin{align}
     P\left(\Xv\in \Xcal_i^*|\DO(X_j=0)\right) 
    = & \alpha p_i \prod_{\ell\in[\lceil m_1\rceil]\backslash\{i,j\}}\overline{p}_\ell +\alb q_i \prod_{\ell\in[\lceil m_1\rceil]\backslash\{i,j\}}\overline{q}_\ell\nonumber\\
    \leq & \frac{\alpha}{m_1}+\alb q_i\prod_{\ell\in[\lceil m_1\rceil]\backslash\{i,j\}} \overline{q}_\ell.
\end{align}
Analogously, since $i\in \mathcal I$, we have 
\begin{align}
     P\left(\Xv\in \Xcal_{i}^*|\DO(X_j=0)\right) 
    \leq & \frac{\alpha}{m_1}
    +\alb q_{i}\prod_{\ell\in[\lceil m_1\rceil]\backslash\{i,j\}} \overline{q}_\ell\nonumber\\
    \leq & \frac{\alpha}{m_1}
    +\alb q_{i}\prod_{\ell = \left\lfloor\frac{m_1}{2}\right\rfloor+1,\ell\neq j}^{\lceil m_1\rceil } \overline{q}_\ell\nonumber\\
    \leq & \frac{\alpha}{m_1}+\alb q_{i}(1-q_{i})^{\lceil m_1\rceil-\left\lfloor\frac{m_1}{2}\right\rfloor-1}\nonumber\\
    \leq & \frac{\alpha}{m_1}+\frac{\alb}{\lceil m_1\rceil -\left\lfloor\frac{m_1}{2}\right\rfloor}\left(1-\frac{1}{\lceil m_1\rceil -\left\lfloor\frac{m_1}{2}\right\rfloor}\right)^{\lceil m_1\rceil-\left\lfloor\frac{m_1}{2}\right\rfloor-1}\nonumber\\
    \leq & \frac{\alpha}{m_1}+\frac{\alb}{m_1} = \frac{1}{m_1}. 
\end{align}
Here, we note that for $m_1>2$, $\lceil m_1\rceil -\lfloor m_1/2\rfloor -1\geq 1$. So we may again use the property that $x(1-x)^n$ is maximized at $x=1/(n+1)$ for $n\geq 1$. 

For $j\in [\lceil m_1\rceil]\backslash\{i\}$, the action $\DO(X_j=1)$ corresponds to 
\begin{equation}
P(\Xv\in \Xcal_i^*|\DO(X_j=1))=0. 
\end{equation}
Note also that for action $\DO(\varnothing)$, $i\in \mathcal I$, we have 
\begin{align}
    P\left(\Xv\in \Xcal_i^*|\DO(\varnothing)\right) = \alpha p_i\prod_{\ell\in[\lceil m_1\rceil]\backslash\{i\}}\overline{p}_\ell+\alb q_i\prod_{\ell\in[\lceil m_1\rceil]\backslash\{i\}}\overline{q}_\ell\leq \frac{1}{m_1}. 
\end{align}
Finally, for $j\in\{m_1+1,\ldots,N-1\}$, the actions $\DO(X_j=1)$ and $\DO(X_j=0)$ correspond to 
\begin{align}
    P\left(\Xv\in \Xcal_i^*|\DO(X_j=1)\right) = P\left(\Xv\in \Xcal_i^*|\DO(X_j=0)\right) = P\left(\Xv\in \Xcal_i^*|\DO(\varnothing)\right)\leq \frac{1}{m_1}. 
\end{align}
Hence, we have completed proving the second part of the lemma.
\end{proof}

\subsection{Proof of Lemma \ref{lem:additional1}}\label{app:additional1}
\begin{proof}
We start by proving the concentration inequality for $\hat{\alpha}$. Recall the definition of $\hat{\alpha}$ in Eq.~\eqref{def:alphahat}. Since $S^{(t)}$, $t\in [T']$ are i.i.d.~Bernoulli random variables, we can apply the Chernoff bound on the sum $\sum_{t\in[T']} S^{(t)}$ and obtain 
\begin{equation}
    P\left( \frac{1}{T'}\sum_{t\in[T']} S^{(t)}\in \left[\alpha\pm \varepsilon_{ 3\alpha,2,T'} \right]\right) \geq 1-\frac{1}{T'}.
\end{equation}

Likewise, for $\hat{\mu}_{\DO(\varnothing)}$ defined in Eq.~\eqref{def:muhatobserve}, by applying Chernoff bound on the sum $\sum_{t\in [T']}Y^{(t)}$, it holds that
\begin{equation}
     P\left(\left|\hat{\mu}_{\DO(\varnothing)}-\mu_{\DO(\varnothing)}\right|\geq \varepsilon_{3,2,T'}\right)\leq \frac{1}{T'}.
\end{equation}

Next, we study $\hat{p}_i$, which is defined in Eq.~\eqref{def:pihat}. We will derive the concentration inequalities for both the denominator and the numerator in Eq.~\eqref{def:pihat}, and further obtain a bound on $\hat{p}_i$ utilizing the union bound. By applying the Chernoff bound on the sum $\sum_{t\in[T']}S^{(t)}$ again (here we choose a different value for $\varepsilon$), we have
\begin{equation}
    P\left( \sum_{t\in[T']} S^{(t)}\in \left[\left(\alpha\pm \varepsilon_{ 3\alpha,2N,T'}\right) T' \right]\right) \geq 1-\frac{1}{NT'}.
\end{equation}
Similarly, since $X_i^{(t)}S^{(t)}$, $t\in [T']$, are i.i.d. Bernoulli random variables, we apply the Chernoff bound and obtain 
\begin{equation}\label{eq:xs}
    P\left( \sum_{t\in[T']} X_i^{(t)}S^{(t)}\in\left[ \left(\alpha p_i\pm \varepsilon_{3\alpha p_i ,2N,T'}\right)T'\right]\right) \geq 1-\frac{1}{NT'}.
\end{equation}

We can combine the above two inequalities using the union bound. In particular, for sufficiently large $T'$ such that $T'\geq \frac{27\log (2NT')}{\alpha}$, we have 
\begin{equation}
    P\left(\hat{p}_i\in \left[\left(1\pm\varepsilon_{{27}/{(\alpha p_i)},2N,T'}\right)p_i\right],\ \forall i\in [N]\right) \geq 1-\frac{2}{T'},
\end{equation}
i.e., $P(E_{\pv})\geq 1-\frac{2}{T'}$
Following the same steps, we can prove the bound for $1-\hat{p}_i$, i.e., $P(E_{\overline{\pv}})\geq 1-\frac{2}{T'}$. 

Finally, we prove the bound for $\hat{\mu}_{s1j1}$, which is defined in Eq.~\eqref{def:muhats1j1}.
If $ \mu_{s1j1} =0$, then $Y^{(t)}=0$ with probability 1, in which case $\hat{\mu}_{s1j1}=0$ with probability 1. This leads to an estimate with zero error. Now we suppose $ \mu_{s1j1} >0$. 
Note that $Y^{(t)}X_j^{(t)}S^{(t)}$, $t\in [T']$, are i.i.d.~and Bernoulli, with Chernoff bound, 
\begin{equation}\label{eq:ystmu}
    P\left(\frac{1}{T'}\sum_{t\in [T']}Y^{(t)}X_j^{(t)}S^{(t)} \in [(1\pm \varepsilon_{27/(\alpha p_j \mu_{s1j1}),2N,T'})\alpha p_j \mu_{s1j1}]\right) \geq 1-\frac{2}{NT'}.
\end{equation}
Combining Eqs.~\eqref{eq:xs} and \eqref{eq:ystmu} using the union bound, we have
\begin{equation}
    P\left(\hat{\mu}_{s1j1}\in\left[  \mu_{s1j1} \pm\varepsilon_{27/(\alpha p_j),2N,T'}\right]\right)\geq 1-\frac{2}{N},
\end{equation}
which completes the proof.

\end{proof}

\subsection{Proof of Lemma \ref{lem:additional2}}\label{app:additional2}

\begin{proof} 

Let $\mathcal C = \left\{j\in[N]:p_j\in\left[\frac{1}{m_1},1-\frac{1}{m_1}\right]\right\}$ be the collection of entries in $\pv$ that are not highly biased either towards zero or one.
By the definition of $m$ in Eq.~\eqref{def:m}, we have $|[N]\setminus\Ccal|\leq m_1$. 

First we show that $\hat{m}_1\leq 2m_1$. 
Since $E_{\pv}$ is true, 
we have 
\begin{equation}
    \hat{p}_j \geq p_j -\varepsilon_{{27p_j}/ \alpha,2N,T'}.
\end{equation}

Note that $1/m_1\leq p_j$ for $j\in\Ccal$.
For sufficiently large $T'$ with $T'> \frac{108m_1\log (2NT')}{\alpha}$ and $j\in \Ccal$, we have $\varepsilon_{{27p_j}/{\alpha},2N,T'}< \frac{p_j}{2}$. 

It follows that 
\begin{equation}\label{eq:pjhatm11}
    \hat{p}_j > \f {p_j}{2}\geq\frac{1}{2m_1},
\end{equation}
for $j\in \Ccal$. 
Similarly, since $E_{\overline{\pv}}$ is true, we have 
\begin{equation}\label{eq:pjhatm12}
    1-\hat{p}_j \geq 1-p_j -\varepsilon_{{27\po_j}/ {\alpha},2N,T'}>\frac{1-p_j}{2}\geq\frac{1}{2m_1},
\end{equation}
for $j\in\Ccal$. 
With Eqs.~\eqref{eq:pjhatm11} and \eqref{eq:pjhatm12}, we have $\hat{p}_j\in\left[\frac{1}{2m_1},1-\frac{1}{2m_1}\right]$ for every $j\in\Ccal$. Hence, the entries in vector $\hat{\pv}$ satisfy the following inequality: 
\begin{equation}
  \left|\left\{j\in[N]:\min\left\{\hat{p}_j,\overline{\hat{p}}_j\right\}<\frac{1}{2m_1}\right\}\right| \leq |[N]\setminus\Ccal|\leq m_1,
\end{equation}
i.e., $|\mathcal I_{2m_1}(\hat{\pv})|\leq m_1$. 
As a result, it follows from the definition of $m$ that  
$\hat{m}_1\leq 2m_1$.

Next, we show that $\hat{m}_1\geq 2m_1/3$. 
Since $E_{\pv}$ is true, if $p_j\leq \frac{1}{m_1}$, we can derive the following inequality analogous to Eq.~\eqref{eq:pjhatm11}:
\begin{equation}
    \hat{p}_j \leq p_j +\varepsilon_{ {27p_j}/{\alpha},2N,T'}<\frac{3}{2m_1}.
\end{equation}

Since $E_{\overline{\pv}}$ is true, if $p_j\geq 1-\frac{1}{m_1}$, we can derive the following bound analogous to Eq.~\eqref{eq:pjhatm12}:
\begin{equation}
    1-\hat{p}_j \leq 1-p_j +\varepsilon_{{27\po_j}/ {\al},2N,T'}<\frac{3}{2m_1}.
\end{equation}
This implies that  
\begin{equation}
    |\{j:\min\{p_j,\overline{p}_j\}\leq \frac{1}{m_1}\}|\leq |\{j:\min\{\hat{p}_j,\overline{\hat{p}}_j\}< \frac{3}{2m_1}\}|.
\end{equation}
With Lem.~\ref{lem:small}, we have 
\begin{equation}
|\{j:\min\{p_j,\overline{p}_j\}\leq \frac{1}{m_1}\}|\geq \left\lceil m_1\right\rceil> \frac{2m_1}{3}.
\end{equation}
Therefore, 
$|\{j:\min\{\hat{p}_j,\overline{\hat{p}}_j\}< \frac{3}{2m_1}\}|> 2m_1/3$, which leads to $\hat{m}_1\geq 2m_1/3$, thus completing the proof.
\end{proof}

\subsection{Proof of Lemma \ref{lem:pjtom1}} \label{app:pjtom1}
\begin{proof}
 From Eq.~\eqref{eq:pj1} in Lem.~\ref{lem:additional1}, if $E_{\pv}$ is true, we have 
\begin{equation}
 \hat{p}_j\leq p_j+\varepsilon_{27/(\alpha p_j),2N,T'}.
 \end{equation}
 for $j\in [N]\setminus\hat{\mathcal B}_{11}$.
 If $T'$ satisfies that $T'> \frac{108m_1\log (2NT')}{\alpha}$,
 we have   
\begin{equation}
    \varepsilon_{27/(\alpha p_j),2N,T'}\leq \frac{1}{2}\sqrt{\frac{p_j}{m_1}}.
\end{equation}
 It follows that, if $E_{\pv}$ is true, 
 \begin{equation}\label{eq:pjhat1}
 \hat{p}_j\leq p_j+\frac{1}{2}\sqrt{\frac{p_j}{m_1}}.
 \end{equation}

With the inequality on $\hat{m}_1$ in Eq.~\eqref{eq:m1bound0} and the definition of $\hat{m}_1$, we have that $\hat{p}_j$ satisfies
\begin{equation}\label{eq:pjhat2}
 \hat{p}_j \geq  \frac{1}{\hat{m}_1}\geq \frac{1}{2m_1}, 
 \end{equation}
 for $j\in [N]\setminus\hat{\mathcal B}_{11}$ if $E_{\pv}$ and $E_{\overline{\pv}}$ are both true. 
With Eq.~\eqref{eq:pj1} in Lem.~\eqref{lem:additional1}, we know that $E_{\pv}$ and $E_{\overline{\pv}}$ are simultaneously true with probability at least $1-\frac{4}{T'}$. 
Combining Eqs.~\eqref{eq:pjhat1} and \eqref{eq:pjhat2}, we get
\begin{equation}
 \frac{1}{2m_1}\leq  p_j+\frac{1}{2}\sqrt{\frac{p_j}{m_1}},
 \end{equation}
for all $j\in [N]\setminus\hat{\mathcal B}_{11}$, with probability at least $1-\frac{4}{T'}$. We can solve for $p_j$ in the inequality above, which yields the desired result. 
\end{proof}

\subsection{Proof of Lemma \ref{lem:additional3}}\label{app:additional3}
\begin{proof}
The expected reward given action $a$ can be expanded as \begin{align}\label{pf:tmp1}
    P_j(Y=1|A=a) 
    = & \frac{\sum_{\mathbf{x}} P_j(Y=1|\Xv=\mathbf{x})P_j(\Xv= \mathbf{x},A=a)}{P_j(A=a)}\nonumber\\
    = & \sum_{\mathbf{x}} r_j(\mathbf{x})P_j(\Xv= \mathbf{x} | A=a)
\end{align}
for $j\in [0:\lceil m_1\rceil]$ and $a\in \Acal$. 
Since changing the reward function does not affect the distribution of $S$ and $\Xv$, the distributions $P_0$ and $P_i$ are identical on $\Xv$ for any $i\in [\lceil m_1\rceil]$. 
As a result, $P_i(\Xv= \mathbf{x} | A=a) = P_0(\Xv= \mathbf{x} | A=a)$ for each $i\in [\lceil m_1\rceil]$. 
From Eq.~\eqref{pf:tmp1} we can express the expected reward under action $a$ as  $P_j(Y=1|A=a) = \Ebb_0 [r_j(\Xv)|A=a]$ for any $j\in[0:\lceil m_1\rceil]$. 
Consequently,  
\begin{align}
    P_j(Y=1|A=a) 
    = & \Ebb_0\left[\frac{1}{2}+\varepsilon\Ibb(\Xv\in \Xcal_j^*)\Big|A=a\right] \nonumber\\
    = & \frac{1}{2}+\varepsilon P_0(\Xv\in\Xcal_j^*|A=a),
\end{align}
which is the desired result. 
\end{proof}

\subsection{Proof of Lemma~\ref{lem:intermediate}}\label{app:intermediate}

\begin{proof}
Suppose the parameters of the system are specified as $(\alpha, \mathbf{p},\mathbf{q},r_{i})$ for $i\in[\lceil m_1\rceil]$. 
Note that the regret $R_T$ satisfies
\begin{align}
    R_T 
    = & \Ebb_{i }[Y |A=a_{i }^*]- \Ebb_{i }\left[\Ebb_{i }[Y=1|A=\hat{A}_T]\right] \nonumber\\
    = & \Ebb_{i }[Y |A=a_{i}^*] - \Ebb_{i } \left[Y|A=a_{i}^*\right] P_{i }(\hat{A}_T=a_{i}^*) - \sum_{a\neq a_{i }^*}\Ebb_{i}[Y|A=a]P_{i}(\hat{A}_T=a),
    \end{align}
    where we decomposed the expected regret into the case that the optimal action $a_i^*$ is chosen and that some other action is selected. With this decomposition, we have 
    \begin{align}
    R_T= & \Ebb_{i } \left[Y|A=a_{i }^*\right] P_{i }(\hat{A}_T\neq a_{i}^*) - \sum_{a\neq a_{i }^*}\Ebb_{i }[Y|A=a]P_{i}(\hat{A}_T=a)\nonumber\\
    \geq & P_{i}(\hat{A}_T\neq a_{i }^*)\left(\Ebb_{i}[Y|A=a_{i}^*]-\max_{a\neq a_{i }^*}\mathbb  E_{i }[Y|A=a]\right)\nonumber\\
    = & \varepsilon P_{i} (\hat{A}_T \neq a_{i}^*)    \left(P_{i}(\Xv\in\Xcal_{i}^*|A=a_{i }^*)-\max_{a\neq a_{i}^*} P_{i}(\Xv\in\Xcal_{i}^*|A=a))\right),
\end{align}
where the last line above follows from Lem.~\ref{lem:additional3}.
Consequently, with Lem.~\ref{lem:1}, 
\begin{equation}
    P_{i}(\Xv\in\Xcal_{i}^*|A=a_{i }^*)-\max_{a\neq a_{i}^*} P_{i}(\Xv\in\Xcal_{i}^*|A=a)) \geq \frac{\alpha}{e}-\frac{1}{m_1},
\end{equation}
which leads to the desired result. 
\end{proof}

\subsection{Proof of Lemma \ref{lem:additional4}}\label{app:additional4}
\begin{proof}
With Lem.~\ref{lemma:kl}, the KL divergence can be expanded as 
\begin{align}
    &D_{\rm KL}(P_0(\Hv_T),P_i(\Hv_T)) \nonumber\\
    = &\sum_{t\in[T]} \sum_{\mathbf{x}^{(t)}\in\Xcal^{N}} P_0\left(\mathbf{x}^{(t)}\right) \left(r_0\left(\mathbf{x}^{(t)}\right)\log\frac{r_0\left(\mathbf{x}^{(t)}\right)}{r_i\left(\mathbf{x}^{(t)}\right)}+\left(1-r_0\left(\mathbf{x}^{(t)}\right)\right)\log\frac{1-r_0\left(\mathbf{x}^{(t)}\right)}{1-r_i\left(\mathbf{x}^{(t)}\right)}\right).
\end{align}
Note that in the inner summation, only $\xv^{(t)}\in \mathcal X_i^*$ can contribute to a non-zero term, since otherwise $r_0 = r_i$. Thus, we only need to consider the elements in $\Xcal_i^*$ in the summation, i.e., 
\begin{align}
& \sum_{\mathbf{x}^{(t)}\in\Xcal^{N}} P_0\left(\mathbf{x}^{(t)}\right) \left(r_0\left(\mathbf{x}^{(t)}\right)\log\frac{r_0\left(\mathbf{x}^{(t)}\right)}{r_i\left(\mathbf{x}^{(t)}\right)}+(1-r_0(\mathbf{x}^{(t)})\log\frac{1-r_0(\mathbf{x}^{(t)})}{1-r_i(\mathbf{x}^{(t)})}\right)\nonumber\\
= & \sum_{\mathbf{x}^{(t)}\in \Xcal_i^*} P_0(\mathbf{x}^{(t)}) \left(r_0(\mathbf{x}^{(t)})\log\frac{r_0(\mathbf{x}^{(t)})}{r_i(\mathbf{x}^{(t)})}+(1-r_0(\mathbf{x}^{(t)})\log\frac{1-r_0(\mathbf{x}^{(t)})}{1-r_i(\mathbf{x}^{(t)})}\right).
\end{align}
With Lem.~\ref{lemma:kl}, since $Y^{(t)}$ is Bernoulli distributed, we can derive the following inequality by upper-bounding the KL divergence with the Chi-square distance (see \cite{tsybakov2008introduction}). For $\xv^{(t)}\in \mathcal X_i^*$, we have  
\begin{align}
& r_0\left(\mathbf{x}^{(t)}\right)\log\frac{r_0\left(\mathbf{x}^{(t)}\right)}{r_i\left(\mathbf{x}^{(t)}\right)}+(1-r_0(\mathbf{x}^{(t)})\log\frac{1-r_0(\mathbf{x}^{(t)})}{1-r_i(\mathbf{x}^{(t)})}\nonumber\\
\leq & \frac{\left(r_0(\mathbf{x}^{(t)})-r_i(\mathbf{x}^{(t)})\right)^2}{r_i(\mathbf{x}^{(t)})(1-r_i(\mathbf{x}^{(t)}))}\nonumber\\
\leq & \frac{\varepsilon^2}{(\frac{1}{2}+\varepsilon)(\frac{1}{2}-\varepsilon)}\nonumber\\
\leq & \frac{16\varepsilon^2}{3},
\end{align}
where the last inequality follows from the fact that $\varepsilon\in (0,1/4)$. Combining the equations above, the KL divergence can be expanded as
\begin{align}
  D_{\rm KL}(P_0(\Hv_T),P_i(\Hv_T)) 
     = &\sum_{t\in [T]} P_0(\Xv^{(t)}\in\mathcal{X}_i^*)\frac{\varepsilon^2}{(\frac{1}{2}+\varepsilon)(\frac{1}{2}-\varepsilon)}\nonumber\\
      \leq & \sum_{t\in [T]}P_0(\Xv^{(t)}\in\mathcal{X}_i^*)\frac{16\varepsilon^2}{3}.
\end{align}
This completes the proof.
\end{proof}

\subsection{Proof of Lemma \ref{lem:miss}}\label{app:miss}
\begin{proof}

Let $T_i = \sum_{t\in[T]}\mathbf{1}_{\{A^{(t)} = a_i^*\}}$ be the number of time steps in which the agent conducts the optimal action throughout the experiment. Then, Eq.~\eqref{eq:klm} can be written as 
\begin{align}
     D_{\rm KL} (P_0(\Hv_T),P_i(\Hv_T))\leq \frac{16\varepsilon^2}{3}\left(1-\frac{1}{m_1}\right) \Ebb_0 [T_i]+\frac{16\varepsilon^2T}{3m_1}.
\end{align}
Recall the definition of set $
\mathcal I = \left\{\pi(j), 1\leq j\leq \left\lfloor\frac{m(\pv) }{2}\right\rfloor\right\}$ introduced in Lem.~\ref{lem:1}. 
Since $\sum_{i\in \mathcal I}T_i \leq T$, we have $\sum_{i\in \mathcal I}\Ebb_0 [T_i]\leq T$. Let $\mathcal I_0$ be a subset of $\mathcal I$ defined as 
\begin{equation}
\mathcal I_0=\left\{i\in \mathcal I: \Ebb_0 [T_i]\leq \frac{2T}{m_1}\right\}.
\end{equation}
It is not difficult to verify that $|\mathcal I_0|\geq \frac{m_1}{2}$. Choose 
\begin{equation}
\varepsilon = \frac{\sqrt{\ln (1.05)}}{4}\sqrt{\frac{m_1}{T}}.
\end{equation}
Note that $\varepsilon\in(0,1/4)$ since $T\geq m_1$.
For each $i\in \mathcal I_0$, we have   
\begin{equation}
     D_{\rm KL} (P_0(\Hv_T),P_i(\Hv_T))\leq \frac{16\varepsilon^2 T}{m_1}= \ln(1.05). 
\end{equation}
With this upper bound on the KL divergence, by Lem.~\ref{lemma:subset}, we have  
\begin{equation}
    P_0(\hat{A}_T=a)+P_i(\hat{A}_T \neq a)\geq \frac{1}{2}\exp\left(-D_{\rm KL}(P_0,P_i)\right)\geq \frac{1}{2.1}
\end{equation}
for any action $a\in \Acal$.
Let $a = a_i^*$. It follows that
\begin{equation}
    P_0(\hat{A}_T=a_i^*)+P_i(\hat{A}_T \neq a_i^*)\geq \frac{1}{2.1}.
\end{equation}
Now we sum the above inequality over $i\in \mathcal I_0$. We obtain 
\begin{align}
    \frac{|\mathcal I_0|}{2.1}
    \leq & \sum_{i\in\mathcal I_0}P_0(\hat{A}_T=a_i^*)+\sum_{i\in\mathcal I_0}P_i(\hat{A}_T \neq a_i^*).
\end{align}
Since the actions $a_i^*$ are distinct, we have 
\begin{equation}
    \sum_{i\in\mathcal I_0}P_0(\hat{A}_T=a_i^*) \leq  1,
\end{equation}
which leads to 
\begin{align}
    \sum_{i\in\mathcal I_0}P_i(\hat{A}_T \neq a_i^*)\geq \frac{|\mathcal I_0|}{2.1}-1.
\end{align}
As a result, there exists $i'\in \mathcal I_0$, such that 
\begin{align}
    P_{i'}\left(\hat{A}_{T} \neq a_{i'}^*\right)\geq \frac{1}{2.1}-\frac{1}{|\mathcal I_0|}\geq \frac{1}{2.1}-\frac{2}{m_1}.
\end{align}
This completes the proof. 
%%%%%
\end{proof}

\end{document}